\documentclass[12pt]{article}
\usepackage[margin=1in]{geometry}

\usepackage{amsfonts, amssymb, amsmath, amsthm, mathtools}
\usepackage{graphicx, multicol, verbatim}
\usepackage{cancel}
\usepackage[shortlabels]{enumitem}
\usepackage{array}
\usepackage{latexsym}
\usepackage{geometry}
\usepackage{fancyhdr}
\usepackage{csquotes}
\usepackage[ruled,vlined]{algorithm2e}
\usepackage{multirow}
\usepackage{subcaption}
\usepackage{caption}
\usepackage{mathrsfs}
\usepackage{color}
\usepackage[utf8]{inputenc}

\usepackage{multirow}
\usepackage{subcaption}
\usepackage{tabularx}
\usepackage{booktabs}

\newcommand{\R}{\mathbb R}
\newcommand{\N}{\mathbb N}

\newcommand{\Network}{\mathcal{N}_{\mathrm{ProductNet}}}

\newtheorem{theorem}{Theorem}[section]

\newtheorem{corollary}[theorem]{Corollary}
\newtheorem{lemma}[theorem]{Lemma}

\newtheorem{definition}[theorem]{Definition}

\definecolor{darkred}{rgb}{1, 0.1, 0.3}
\definecolor{darkgreen}{rgb}{0.5, 0.8, 0.1}
\definecolor{darkpurple}{rgb}{1.0, 0, 1.0}
\definecolor{darkblue}{rgb}{0, 0, 1.0}

\newcommand{\ourfunc}  {{symmetric and factor-wise group invariant function}}
\newcommand{\SFGI} {{SFGI}}
\usepackage[most]{tcolorbox}
\usepackage{authblk}
\usepackage{expl3}
\usepackage{varioref}
\usepackage{hyperref}
\usepackage{cleveref}
\usepackage[round]{natbib}
\begin{document}

\title{Neural approximation of Wasserstein distance via a universal architecture for symmetric and factorwise group invariant functions}

\author[1]{ Samantha Chen \thanks{\url{sac003@ucsd.edu}}}
\author[1, 2]{Yusu Wang \thanks{\url{yusuwang@ucsd.edu}}}

\affil[1]{Department of Computer Science and Engineering, University of California - San Diego }
\affil[2]{Hal{\i}c{\i}o\u{g}lu Data Science Institute, University of California - San Diego }

\maketitle
\maketitle

\begin{abstract}
Learning distance functions between complex objects, such as the Wasserstein distance to compare point sets, is a common goal in machine learning applications. However, functions on such complex objects (e.g., point sets and graphs) are often required to be invariant to a wide variety of group actions e.g. permutation or rigid transformation. 
Therefore, continuous and symmetric \textit{product} functions (such as distance functions) on such complex objects must also be invariant to the \textit{product} of such group actions. 
We call these functions \emph{\ourfunc{}s} (or \emph{\SFGI{} functions} in short).
In this paper, we first present a general neural network architecture for approximating \SFGI{} functions. The main contribution of this paper combines this general neural network with a sketching idea to develop a specific and efficient neural network which can approximate the $p$-th Wasserstein distance between point sets.
Very importantly, the required model complexity is \emph{independent} of the sizes of input point sets. 
On the theoretical front, to the best of our knowledge, this is the first result showing that there exists a neural network with the capacity to approximate Wasserstein distance with bounded model complexity. Our work provides an interesting integration of sketching ideas for geometric problems with universal approximation of symmetric functions. 
On the empirical front, we present a range of results showing that our newly proposed neural network architecture performs comparatively or better than other models (including a SOTA Siamese Autoencoder based approach). In particular, our neural network generalizes significantly better and trains much faster than the SOTA Siamese AE.
Finally, this line of investigation could be useful in exploring effective neural network design for solving a broad range of geometric optimization problems (e.g., $k$-means in a metric space).
\end{abstract}

\section{Introduction}
\label{section:introduction}
Recently, significant interest in geometric deep learning has led to a focus on neural network architectures which learn functions on complex objects such point clouds \citep{zaheer2017deep, qi2017pointnet} and graphs \citep{scarselli2008graph}. 
Advancements in the development of neural networks for complex objects has led to progress in a variety of applications from 3D image segmentation \citep{uy2019revisiting} to drug discovery \citep{bongini2021molecular, gilmer2017neural}.
One challenge in learning functions over such complex objects is that the desired functions are often required to be invariant to certain group actions.
For instance, functions on point clouds are often permutation invariant with respect to the ordering of individual points.
Indeed, developing permutation invariant or equivariant neural network architectures, as well as understanding their universal approximation properties, has attracted significant attention in the past few years; see e.g., \citep{qi2017pointnet, zaheer2017deep, maron2018invariant, maron2019universality, segol2019universal, keriven2019universal, yarotsky2022universal, wagstaff2019limitations, wagstaff2022universal, bueno2021representation}

However, in many geometric or graph optimization problems, our input goes beyond a single complex object, but multiple complex objects. 
For example, the $p$-Wasserstein distance $\mathrm{W}_p(X, Y)$ between two point sets $X$ and $Y$ sampled from some metric space (e.g., $\mathbb{R}^d$) is a function over pairs of point sets. 
To give another example, the $1$-median of the collection of $k$ point sets $P_1, \dots, P_k$ in $\R^d$ can be viewed as a function over $k$ point sets. 

A natural way to model such functions is to use product space.
In particular, let $\mathcal{X}$ denote the space of finite point sets from a bounded region in $\mathbb{R}^d$. 
Then the $p$-Wasserstein distance can be viewed as a function $\mathrm{W}_p: \mathcal{X} \times \mathcal{X} \to \mathbb{R}$. 
Similarly, 1-median for $k$ point sets can be modeled by a function from the product of $k$ copies of $\mathcal{X}$ to $\R$. 
Such functions are not only invariant to permutations of the factors of the product space (i.e. $\mathrm{W}_p(X, Y) = \mathrm{W}_p(Y, X)$) but are also invariant or equivariant with respect to certain group actions for each factor. 
For $p$-Wasserstein distance, $\mathrm{W}_p$ is invariant to permutations of the ordering of points within both $X$ and $Y$.
This motivates us to extend the setting of learning continuous group invariant functions to learning continuous functions over \textit{product} spaces which are both  invariant to some product of group actions and symmetric. 
More precisely, we consider a type of function which we denote as an \SFGI{} product functions. 
\begin{definition}[SFGI product function]\label{def:SFGI}
Given compact metric spaces $(\mathcal{X}_i, \mathrm{d}_{\mathcal{X}_i})$ where $i \in [k]$,  we define \textit{symmetric and factor-wise group invariant} (\SFGI) product functions as $f: \mathcal{X}_1 \times \mathcal{X}_2 \times \cdots \mathcal{X}_k \to \R$ where $f$ is (1) symmetric over the $k$ factors, and (2) invariant to the group action $G_1 \times G_2 \times \cdots \times G_k$ for some group $G_i$ acting on $\mathcal{X}_i$, for each $i\in [1, k]$.
\end{definition} 

\paragraph{Contributions.}\SFGI{} product functions can represent a wide array of geometric matching problems including computing the Wasserstein or Hausdorff distance between point sets.
In this paper, we first provide a general framework for approximating \SFGI{} product functions in Section \ref{subsec:general-framework}. 
Our primary contribution, described in Section \ref{subsec:point-set-functions}, is the integration of this general framework with a sketching idea in order to develop an efficient and specific \SFGI{} neural network which can approximate the $p$-Wasserstein distance between point sets (sampled from a compact set in a nice  metric space, such as the fixed-dimensional Euclidean space).
Most importantly, the complexity of our neural network (i.e., number of parameters needed) is \emph{independent} of the maximum size of the input point sets, and depends on only the additive approximation error. 
To the best of our knowledge, this is the first architecture which provably achieves this property.
This is in contrast to many existing universal approximation results on graphs or sets (e.g., for DeepSet) where the network sizes depend on the size of each graph or point set in order to achieve universality \citep{maron2019universality, wagstaff2022universal, bueno2021representation}. 
We also provide a range of experimental results in Section \ref{section:experiments} showing the utility of our neural network architecture for approximating Wasserstein distances. 
We compare our network with both a SOTA Siamese autoencoder \citep{kawano2020learning}, a natural Siamese DeepSets network, and the standard Sinkhorn approximation of Wasserstein distance. 
Our results show that our universal neural network architecture produces Wasserstein approximations which are better than the Siamese DeepSets network, comparable to the SOTA Siamese autoencoder and generalize much better than both to input point sets of sizes which are unseen at training time. Furthermore, we show that our approximation (at inference time) is much faster than the standard Sinkhorn approximation of the $p$-Wasserstein distance at similar error threshholds.
Moreover, our neural network trains much faster than the SOTA Siamese autoencoder. 
Overall, our network is able to achieve \textbf{equally accurate or better} Wasserstein approximations which \textbf{generalize better} to point sets of unseen size as compared to SOTA while \textbf{significantly reducing} training time. 
In Appendix \ref{subsec:comparisonNNs}, we provide discussion of issues with other natural choices of neural network architectures one might use for estimating Wasserstein distances, including one directly based on Siamese networks, which are often used for metric learning. 

All missing proofs / details can be found in the {\bf Supplementary materials}. 


\paragraph{Related work.} Efficient approximations of Wasserstein distance via neural networks are an active area of research. 
One approach is to use input convex neural networks to approximate the 2-Wasserstein distance\citep{makkuva2020optimal, taghvaei20192}. However, for this approach, training is done \textit{per} pair of inputs and is restricted to the 2-Wasserstein distance  which makes it unsuitable for a general neural network approximation of $p$-Wasserstein distances between discrete distributions. 
This neural approximation method contrasts with our goal: a general neural network that can approximate the $p$-Wasserstein distance between any two point sets in a compact metric space to within $\epsilon$-accuracy. 
Siamese networks are another approach for popular approach for learning Wasserstein distances. Typically, a Siamese network is composed of a single neural network which maps two input instances to Euclidean space. The output of the network is represented then by $\ell_p$-norm between the output embeddings. In \citep{courty2017learning}, the authors utilize a Siamese autoencoder which takes two histograms (images) as input. For their architecture, a single encoder network is utilized to map each histogram to an embedding in Euclidean space, while a decoder network maps each Euclidean embedding back to an output histogram. 
The Kullback-Liebler (KL) divergence between original histogram and the output histogram (i.e. reconstruction loss) is used during training to regularize the embeddings and the final estimate of 2-Wasserstein distance is the $\ell_2$ norm between the embeddings.
The idea of learning Wasserstein distances via Siamese autoencoders was extended in \citep{kawano2020learning} to point cloud data with the Wasserstein point cloud embedding network (WPCE) where the original KL reconstruction loss was replaced with a differentiable Wasserstein approximation between the original point set and a fixed-size output point set from the decoder network. 
In our subsequent experiments, we show that our neural network trains much more efficiently and generalizes much better than WPCE to point sets of unseen size.

Moreover, the concept of group invariant networks was previously investigated in several works, including \citep{zaheer2017deep, qi2017pointnet, maron2018invariant, lim2022sign, deng2021vector}. 
For instance, DeepSets \citep{zaheer2017deep} and PointNet \citep{maron2018invariant} are two popular permutation invariant neural network which were shown to be universal with respect to set functions. 
In addition to group invariance, there have also been efforts to explore the notion of invariance with respect to combinations of groups, such as invariance to both SE(3) and permutation group \citep{du2022se, maron2020learning} or combining basis invariance with permutation invariance \citep{lim2022sign}.
Our work differs from previous work in that we address a universal neural network which is invariant with respect to a \textit{specific} combination of permutation groups that corresponds to an \SFGI{} function on point sets. 
In general, we can view this as a subgroup of the permutation group - encoding the invariance of each individual point set with the symmetric requirement of the product function corresponds to a specific subgroup of the permutation group.
Thus, previous results regarding permutation invariant architectures such as DeepSets, PointNet or combinations of group actions (such as \citep{du2022se}) do not address our setting of \SFGI{} functions or $p$-Wasserstein distances.

\section{Preliminaries}
\label{section:preliminaries}
We will begin with basic background on groups, universal approximation, and Wasserstein distances. 
\paragraph{Groups and group actions}
A group $G$ is an algebraic structure that consists of a set of elements and a binary operation that satisfies a specific set of axioms: (1) the associative property, (2) the existence of an identity element, and (3) existence of inverses for each element in the set. 
Given a metric space $(\mathcal{X}, \mathrm{d}_\mathcal{X})$, the action of the group $G$ on $\mathcal{X}$ is a function $\alpha: G \times \mathcal{X} \to \mathcal{X}$ that transforms the elements of $\mathcal{X}$ for each element $\pi \in G$. 
For each element $\pi \in G$, we will write $\pi \cdot x$ to denote the action of a group element $\pi$ on $x$ instead of $\alpha(\pi, x)$.
For example, if $G$ is the permutation group over $[N]:=\{1, 2, \ldots, N\}$, and $\mathcal{X} = \R^N$, then for any $\pi \in G$, $\pi \cdot x$ represents the permutation of elements in $x\in \R^N$ via $\pi$ i.e. given $x = (x_1, x_2, \dots, x_N)$, $\pi \cdot x = (x_{\pi(1)}, x_{\pi(2)} \dots, x_{\pi(N)})$. 
A function $f: \mathcal{X} \to \mathcal{Y}$ is \emph{$G$-invariant} if for any $X \in \mathcal{X}$ and any $\pi \in G$, we have that $f(X) = f(\pi \cdot X)$.  

\paragraph{Universal Approximation.}
Let $\mathcal{C}(\mathcal{X}, \R)$ denote the set of continuous functions from a metric space $(\mathcal{X}, d_{\mathcal{X}})$ to $\R$. Given two families of functions $\mathcal{F}_1$ and $\mathcal{F}_2$ where $\mathcal{F}_1 \subseteq \mathcal{F}_2$ and $\mathcal{F}_1, \mathcal{F}_2 \subseteq \mathcal{C}(\mathcal{X}, \R)$, we say that $\mathcal{F}_1$ \textit{universally approximates} $\mathcal{F}_2$ if for any $\epsilon> 0$ and any $f \in \mathcal{F}_2$, there is a $g \in \mathcal{F}_1$ such that $\|g - f\|_\infty < \epsilon$. Different norms on the space of functions can be used, but we will use $L_\infty$ norm in this paper, which intuitively leads to additive pointwise-error over the domain of these functions.
The classic universal approximation theorem for multilayer perceptrons (MLPs) \citep{cybenko1989approximation} states that a feedforward neural network with a single hidden layer, using certain activation functions, can approximate any continuous function to within an arbitrary \emph{additive $\epsilon$-error.} 

\paragraph{Permutation invariant neural networks for point cloud data}
\label{subsec:point-cloud-invariance}
One of the most popular permutation invariant neural network models is the DeepSets model defined in \citep{zaheer2017deep}. 
DeepSets is designed to handle unordered input point sets by first applying a neural network to each individual element, then using sum-pooling to generate an embedding for the input data set, and finally, applying a final neural network architecture to the ouput embedding. 
Formally, suppose we are given a finite \emph{multiset} $S = \{x : x\in \R^d\}$ (meaning that an element can appear multiple times, and the number of times an element occurs in $S$ is called its \emph{multiplicity}). 
The DeepSets model is defined as 
\begin{equation*}
    \mathcal{N}_{\mathrm{DeepSet}}(S) = g_{\theta_2} \Big(\sum_{x \in S} h_{\theta_1}(x)\Big)
\end{equation*}
where $h_{\theta_1}$ and $g_{\theta_2}$ are neural network architectures. 
DeepSets can handle input point sets of variable sizes. 
It was also shown to be universal with respect to continuous \emph{multiset functions}. 
\begin{theorem}[\citep{zaheer2017deep, amir2023neural, dym2022low}]
\label{thm:sum-decomposition}
    Assume the elements are from a compact set in $\R^k$, 
    and the input multiset size is fixed as $N$. Let $t=2kN + 1$. Then any continuous multiset function, represented as $f: \R^{k \times N} \to \R$ which is invariant with respect to permutations of the columns, can be approximated arbitrarily close in the form of $\rho \Big(\sum_{x \in X} \phi(x) \Big)$, for continuous transformations $\phi:\R^k \to \R^{t}$ and $\rho: \R^t \to \R$.

\end{theorem}

While universality for the case when $k = 1$ was shown using symmetric polynomials, the case for $k > 1$ in fact is quite subtle and the proof in \citep{zaheer2017deep} misses key details. For completeness, we provide a full proof in Appendix \ref{appendix:deepsets} for when the output dimension of $\phi$ is $t = \binom{k + N}{k}$. It was recently shown in \citep{dym2022low, amir2023neural} that the output dimension of $\phi$ can be reduced to $2kN + 1$, which is the dimension of $t$ which we use in Theorem 2.1 and subsequent theorems. In both the cases where the output dimension of $\phi$ is $t = \binom{k + N}{k}$ or $t = 2kN + 1$, Theorem 2.1 implies that if we want to achieve universality, the required network size depends on input point cloud size.

\paragraph{Wasserstein distances and approximations.}
Here we will introduce Wasserstein distance for discrete measures. Let $(X, \mathrm{d}_X)$ be a metric space. 
For two weighted point sets $P = \{(x_i, w_i) : x_i \in X, \sum_{w_i} = 1, i \in [n]\}$ and $Q = \{(x_i', w_i'): x_i' \in X, \sum_{w_i} = 1, i \in [m]\}$, we define the Wasserstein distance between $P$ and $Q$ as 
\begin{equation*}
    \mathrm{W}_p(P, Q) = \min_{\Pi \in \R_+^{n \times m}} \Big\{\Big(\langle \Pi, D^p\rangle\Big)^{1/p} : \Pi \mathbf{1} = [w_1, \dots, w_n], \Pi^T \mathbf{1} = [w_1', \dots, w_m'] \Big\}
\end{equation*}
where $D \in \R_{+}^{n \times m}$ is the distance matrix with $D_{i, j} = \mathrm{d}_X(x_i, x_j')$.
One can think of these weighted point sets as discrete probability distributions in $(X, \mathrm{d}_X)$.
When $p = 1$, $\mathrm{W}_1$ is also commonly known as the Earth Mover's distance (EMD). 
Additionally, note that when $p = \infty$, $\mathrm{W}_p$ is the same as the Hausdorff distance between points in $P$ and $Q$ with non-zero weight. 
Computing Wasserstein distances amounts to solving a linear programming problem, which takes $O(N^3 \log N)$ (where $N = \max\{n, m\}$) time. 
There have been a number of methods for fast approximations of Wasserstein distances, including  multi-scale and hierarchical solvers \citep{schmitzer2016sparse}, and $L_1$ embeddings via quadtree algorithms \citep{backurs2020scalable, IndykThaper2003}. In particular, entropic regularization of Wasserstein distance \citep{cuturi2013sinkhorn}, also known as the Sinkhorn distance, is often used as the standard Wasserstein distance approximation for learning tasks. Unlike Wasserstein distance, the Sinkhorn approximation is differentiable and can be computed in approximately $O(n^2)$ time. 
The computation time is governed by a regularization parameter $\epsilon$. 
As $\epsilon$ approaches zero, the Sinkhorn distance approaches the true Wasserstein distance.

\section{Learning functions between point sets}
We will first present a general framework for approximating \SFGI-functions and then show how this framework along with geometric sketches of our input data enables us to define neural networks which can approximate $p$-Wasserstein distances with complexity independent of input data size.
\subsection{A general framework for functions on product spaces}
\label{subsec:general-framework}

One of the key ingredients in our approach is the introduction of what we call a \emph{sketch} of input data to an Euclidean space whose dimension is independent of the size of the input data.

\begin{definition}[Sketch]
    Let $\delta > 0$, $a \in \N^+$, and $G$ be a group which acts on $\mathcal{X}$. A \textbf{$(\delta, a, G)$-sketch} of a metric space $(\mathcal{X}, \mathrm{d}_\mathcal{X})$ consists of a $G$-invariant continuous encoding function $h: \mathcal{X} \to \R^a$ and a continuous decoding function $g: \R^a \to \mathcal{X}$ such that $\mathrm{d}_\mathcal{X}(g \circ h(S), S) < \delta$.
\end{definition}

Now let $(\mathcal{X}_1, \mathrm{d}_{\mathcal{X}_1}), \dots, (\mathcal{X}_m, \mathrm{d}_{\mathcal{X}_m})$ be compact metric spaces. 
The product space $\mathcal{X}_1 \times \cdots \times \mathcal{X}_m$ is still a metric space equipped with the following natural metric induced from metrics of each factor: 
    \begin{equation*}
        \mathrm{d}_{\mathcal{X}_1 \times \cdots \times \mathcal{X}_m}((A_1, \dots, A_m), (A_1', \dots, A_m')) = \mathrm{d}_{\mathcal{X}_1}(A_1, A_1') + \cdots + \mathrm{d}_{\mathcal{X}_m}(A_m, A_m').
    \end{equation*} 
Suppose $G_i$ is a group acting on $\mathcal{X}_i$, for each $i\in [m]$. 
In the following result, instead of \SFGI{} product functions, we first consider the more general case of \emph{factor-wise group invariant functions}, namely functions $f: \mathcal{X}_1 \times \cdots \times \mathcal{X}_m \to \R$ such that $f$ is uniformly continuous and invariant to the group action $G_1 \times \cdots \times G_m$. 
\begin{lemma}
\label{thm:product-network}
    Suppose $f: \mathcal{X}_1 \times \cdots \times \mathcal{X}_m \to \R$ is uniformly continuous and invariant to $G_1 \times \cdots \times G_m$. 
    Additionally, assume that for any $\delta > 0$, $(\mathcal{X}_i, \mathrm{d}_{\mathcal{X}_i})$ has a $(\delta, a_i, G_i)$-sketch where $a_i$ may depend on $\delta$. 
    Then for any $\epsilon > 0$, there is a continuous $G_i$-invariant functions $\phi_i: \mathcal{X}_i \to \R^{a_i}$ for all $i \in [k]$ and a {continuous} function $\rho: \R^{a_1} \times \cdots \times \R^{a_m} \to \R$ such that 
    \begin{equation*}
        |f(A_1, A_2, \dots, A_m) - \rho(\phi_1(A_1), \phi_2(A_2), \dots, \phi_k(A_m)) | < \epsilon
    \end{equation*}
    for any $(A_1, \dots, A_m) \in \mathcal{X}_1 \times \cdots \times \mathcal{X}_m$.
    \noindent Furthermore, if $\mathcal{X}_1= \ldots \mathcal{X}_2 = \cdots = \mathcal{X}_m$, then we can choose $\phi_1 = \phi_2 = \ldots \phi_m$. 
\end{lemma}

Note that a recent result from \citep{lim2022sign} shows that a continuous factor-wise group invariant function $f: \mathcal{X}_1 \times \cdots \mathcal{X}_m \to \R$ can be {\bf represented} (not approximated) by the form $f(v_1, \dots, v_m) = \rho(\phi_1(v_1), \dots, \phi_k(v_m))$ if \emph{there exists a topological embedding from $\mathcal{X}_i/G_i$ to Euclidean space}. The condition that each quotient $\mathcal{X}_i/G_i$ has a topological embedding in fixed dimensional Euclidean space is strong. A topological embedding requires injectivity, while in a sketch, one can collapse input objects as long as after decoding, we obtain an approximated object which is close to the input.
Our result can be viewed as a relaxation of their result by allowing our space to have an approximate fixed-dimensional embedding (i.e., our {\bf $(\delta, a, G)$-sketch}). 

We often consider the case where $\mathcal{X} = \mathcal{X}_1 =  \cdots = \mathcal{X}_m$ i.e. $f: \mathcal{X} \times \cdots \mathcal{X} \to \R$ where $G$ is a group acting on the factor $\mathcal{X}$. 
Oftentimes, we require the function to not only be invariant to the actions of a group $G$ on each individual $\mathcal{X}$ but also \emph{symmetric} with respect to the ordering of the input. 
By this, we mean $f(A_1, \dots, A_m) = f(A_{\pi(1)}, \dots, A_{\pi(m)})$ where $\pi$ is a permutation on $[m]$. 
In other words, we now consider the \emph{\SFGI{} product function $f$} as introduced in Definition \ref{def:SFGI}. 
The extra symmetry requirement adds more constraints to the form of $f$.
We show that the set of uniformly continuous \SFGI{} product function can be universally approximated by product function with an even simpler form than Lemma \ref{thm:product-network} as stated in the theorem below.


\begin{lemma}
\label{thm:deep-set-final-step}
    Assume the same setup as Lemma \ref{thm:product-network} with $\mathcal{X} = \mathcal{X}_1 = \cdots = \mathcal{X}_m$ and $G = G_1 = \cdots = G_m$. Assume that $\mathcal{X}$ has a $(\delta, a, G)$-sketch. Additionally, suppose $f$ is symmetric; hence $f$ is a \SFGI{} function. Let $t = 2am + 1$. Then for any $\epsilon > 0$, there is a continuous $G$-invariant function $\phi: \mathcal{X} \to \R^t$ and a {continuous} function $\rho: \R^t \to \R$ such that 
    \begin{equation*}
        |f(A_1, \dots, A_m) - \rho\Big(\sum_{i = 1}^m \phi(A_i)\Big)| < \epsilon
    \end{equation*}
\end{lemma}
Now suppose we want approximate an \SFGI{} product function, $f$, with a neural network. 
Lemma \ref{thm:deep-set-final-step} implies that we can approximate $\phi$ with any universal $G$-invariant neural network which embeds our original space $\mathcal{X}$ to some Euclidean space $\R^a$. 
Then the outer architecture $\rho$ can be any universal architecture (e.g. MLP). 
Finding a universal $G$-invariant neural network to realize $\phi$ over a single factor space $\mathcal{X}$ is in general much easier than finding a \SFGI{} neural network, and as we discussed at the end of Section \ref{section:introduction}, we already know how to achieve this for several settings. 
We will show how this idea is at work for approximating \SFGI{} functions between point sets in the next subsection.

\subsection{Universal neural networks for functions between point sets}
\label{subsec:point-set-functions}

We are interested in learning symmetric functions between point sets (i.e. any $p$-Wasserstein distance) which are factor-wise \textit{permutation} invariant. 
In this section, we will show that we can find a $(\delta, a, G)$-sketch for the space of weighted point sets. This allows us to combine Lemma \ref{thm:deep-set-final-step} with DeepSets to define a set of neural networks which can approximate $p$-th Wasserstein distances to arbitrary accuracy. Furthermore, the encoding and decoding functions can be approximated with neural networks where their model complexity is \textit{independent} of input point set size. Thus, the resulting neural network used to approximate Wasserstein distance also has {\bf bounded model complexity}. 


\paragraph{Set up.} Given some metric space $(\Omega, \mathrm{d}_\Omega)$, let $\mathcal{X}$ be the set of {\it weighted} point sets with at most $N$ elements.
In other words, each $S \in \mathcal{X}$ has the form $S = \{(x_i, w_i) : w_i \in \R_+, \sum_i w_i = 1, x_i \in \Omega\}$ and $|S| \leq N$. One can also consider $\mathcal{X}$ to be the set of weighted Dirac measures over $\Omega$. For simplicity, we also sometimes use $S$ to refer to just the set of points $\{x_i\}$ contained within it. 
We will consider the metric over $\mathcal{X}$ to be the $p$-th Wasserstein distance, $\mathrm{W}_p$. We refer to the metric space of weighted point sets over $\Omega$ as $(\mathcal{X}, \mathrm{W}_p)$.



First, we will show that given a $\delta$, there is a $(\delta, a, G)$-sketch of $\mathcal{X}$ with respect to $\mathrm{W}_p$. 
The embedding dimension $a$ depends on the so-called \emph{covering number} of the metric space $(\Omega, \mathrm{d}_\Omega)$ from which points are sampled. 
Given a compact metric space $(\Omega, \mathrm{d}_\Omega)$, the \emph{covering number $\nu_\Omega(r)$ w.r.t. radius $r$} is the minimal number of radius $r$ balls needed to cover $\Omega$. As a simple example, consider $\Omega = [-\Delta, \Delta] \subseteq \R$. Given any $r$, we can cover $X$ with $\frac{2\Delta}{r}$ intervals so $\nu_\Omega(r) \leq \frac{2\Delta}{r}$. 
The collection of the center of a set of $r$-balls that cover $\Omega$ an \emph{$r$-net of $\Omega$}. 
For a compact set $\Omega \subset \R^d$ with diameter $D$, its covering number $\nu_\Omega(r)$ is a constant depending on $D$, $r$ and $d$ only.

\begin{theorem}
\label{lemma:point-sketch}
    Set $\mathrm{d}_{\mathcal{X}}$ to be $W_p$ for $1 \leq p < \infty$. Let $G$ be the permutation group. 
    For any $\delta > 0$, let $\delta_0 = \frac{1}{2}\sqrt[p]{\delta/2}$ and $a = \nu_\Omega(\delta_0)$ be the covering number w.r.t. radius $\delta_0$. 
    Then there is a $(\delta, a, G)$-sketch of $\mathcal{X}$ with respect to $\mathrm{W}_p$. 
    Furthermore, the encoding function $\mathbf{h}: \mathcal{X} \to \R^a$ can be expressed as the following where $h: \Omega \to \R^a$ is continuous: 
    \begin{equation}\label{eqn:hdecomp} 
        \mathbf{h}(S) = \sum_{x \in S} h(x) . 
    \end{equation}
\end{theorem}

\begin{proof}
    Let $\delta > 0$ and let $S \in \mathcal{X}$ be $S = \{(x_i, w_i) : \sum w_i = 1, x_i \in \Omega\}$ and $|S| \leq N$. Given $\delta_0 = \frac{1}{2}\sqrt[p]{\delta/2}$ and $a = \nu_\Omega(\delta_0)$, we know $\Omega$ has a $\delta_0$-net, $C$, and we denote the elements of $C$ as  $\{y_1, \dots, y_a\}$. In other words, for any $x \in \Omega$, there is a $y_i \in Cd$ such that $\mathrm{d}_\Omega(x, y_i) < \delta_0$. 
    
    First, we will define an encoding function $\rho: \mathcal{X} \to \R^a$.  For each $y_i$, we will use a soft indicator function $e^{-b \mathrm{d}_{\Omega}(x , B_{\delta_0}(y_i))}$ and set the constant $b$ so that $e^{-b \mathrm{d}_\Omega(x, B_{\delta_0}(y_i))}$ is "sufficiently" small if $\mathrm{d}_\Omega(x, B_{\delta_0}(y_i)) > \delta_0$. More formally, we know that $\lim_{b \to \infty} e^{-b \delta_0} = 0$ so there is $\beta \in \R$ such that for all $b > \beta$, $e^{-b \delta_0} < \frac{\delta_0^p}{d_{max}^p \cdot a}$. Set $b_0$ to be such that $b_0 > \beta$. Let $h_i(x) = e^{-b_0 \mathrm{d}_\Omega(x, B_{\delta_0}(y_i))}$ for each $i \in [a]$. For a given $x \in \Omega$, we compute $h: \Omega \to \R^a$ as 
    \begin{equation*}
        h(x) = [h_1(x), \dots, h_a(x)]
    \end{equation*}
    Then we define the encoding function $\mathbf{h}: \mathcal{X} \to \R^a$ as 
    \begin{equation*}
        \mathbf{h}(S) = \sum_{i = 1}^n w_i \frac{h(x_i)}{\|h(x_i)\|_1}
    \end{equation*}
    Note that $\|\mathbf{h}(S)\|_1 = 1$ and $\mathbf{h}$ is continuous since Wasserstein distances metrize weak convergence. Additionally, since $\mathrm{d}_{\Omega}(x, B_{\delta_0}(y_i))$ is the distance from $x$ to the $\delta_0$-ball around $y_i$, we are guaranteed to have one $j$ where $h_j(x_i) = 1$ so $\|h(x_i)\|_1 > 1$.

    Now, we define a decoding function $g: \R^a \to \mathcal{X}$ as $g(v) = \{(y_i, \frac{v_i}{\|v\|_1}): i \in [a]\}$. In order to show that $g$ and $\mathbf{h}$ yields a valid $(\delta, a, G)$-sketch of $\mathcal{X}$, we must show that $g \circ \mathbf{h}(S)$ is sufficiently close to $S$ under the $\mathrm{W}_p$ distance. First, we know that
    \begin{equation*}
        \mathrm{W}_p^p(g \circ \mathbf{h}(S), S) \leq \sum_{i = 1}^n \sum_{j = 1}^a w_1 \frac{h_j(x_i)}{\|h(x_i)\|_1} d(x_i, y_j)^p.
    \end{equation*}
    Let $d_{max}$ be the diameter of $\Omega$. For a given $x_i$, we can partition $\{h_1(x_i), \dots, h_a(x_i)\}$ into those where $h_j(x_i) \geq \frac{\delta_0^p}{d_{max}^p \cdot a}$ and those where $h_j(x_i) < \frac{\delta_0^p}{d_{max}^p \cdot a}$ i.e. $\{h_{j_1}(x_i), \dots, h_{j_k}(x_i)\}$ and $\{h_{j_{k + 1}}(x_i), \dots, h_{j_a}(x_i)\}$ respectively. If $h_j(x) \geq \frac{\delta_0^p}{d_{max}^p \cdot a}$, then 
    \begin{equation*}
        e^{-b_0 \mathrm{d}_\Omega(x, B_{\delta_0}(y_i))} \geq \frac{\delta_0^p}{d_{max}^p \cdot a} > e^{-b_0 \delta_0}
    \end{equation*}
    so $\mathrm{d}_{\Omega}(x, B_{\delta_0}(y_i)) < \delta_0$. Then 
    \begin{align*}
        \mathrm{W}_p^p(g \circ \mathbf{h}(S), S) &\leq \sum_{i = 1}^n \sum_{j = 1}^m w_i \frac{h_j(x_i)}{\|h(x_i)\|_1} \mathrm{d}_\Omega(x_i, y_j)^p.\\
        &< \sum_{i = 1}^n w_i \Big( \sum_{\ell = 1}^k \frac{h_{j_\ell}(x_i)}{\|h(x_i)\|_1}(2\delta_0)^p + \sum_{\ell = k + 1}^{a} \frac{\delta_0^p}{d_{max}^p \cdot a} d_{max}^p \Big) \text{ since } \mathrm{d}_\Omega(x_i, y_j) \leq d_{max} \\
        &\leq \sum_{i = 1}^n w_i(2^p\delta_0^p + \delta_0^p) \leq 2^p \Big(\sqrt[p]{\delta/2} \cdot \frac{1}{2}\Big)^p + \frac{1}{2^p} \cdot \frac{\delta}{2} < \frac{\delta}{2} + \frac{\delta}{2} = \delta
    \end{align*}
    Thus, the encoding function $\mathbf{h}$ and the decoding function $g$ make up a $(\delta, a, G)$-sketch. 
\end{proof}

Note that the sketch outlined in Theorem 3.4 is a smooth version of a one-hot encoding. With Theorem \ref{lemma:point-sketch} and Lemma \ref{thm:deep-set-final-step}, we will now give an explicit formulation of an $\epsilon$-approximation of $f$ via sum-pooling of continuous functions. 
\begin{corollary}
\label{corollary:product-net-universality}
    Let $\epsilon > 0$, $(\Omega, d_\Omega)$ be a compact metric space and let $\mathcal{X}$ be the space of weighted point sets equipped with the $p$-Wasserstein, $\mathrm{W}_p$. Suppose for any $\delta$, $(\Omega, d_\Omega)$ has covering number $a(\delta)$. Then given a function $f: \mathcal{X} \times \mathcal{X} \to \R$ that is uniformly continuous and permutation invariant, there is continuous functions $h: \Omega \to \R^{a(\delta)}$, $\phi: \R^{a(\delta)} \to \R^{a'}$, and $\rho: \R^{a'} \to \R$, such that for any $A, B \in \mathcal{X}$
    \begin{equation*}
        \Bigg|f(A, B) - \rho\Big(\phi\Big(\sum_{(x, w_x) \in A} w_xh(x)\Big) + \phi\Big(\sum_{(x, w_x) \in B}w_xh(x)\Big)\Big) \Bigg| < \epsilon
    \end{equation*}
    where $h, \phi$ and $\rho$ are all continuous and $a' = 4 \cdot a(\delta) + 1$. 
\end{corollary}

Due to Eqn. (\ref{eqn:hdecomp}), instead of considering the function $\mathbf{h}$ which takes a set of points $S \in \mathcal{X}$ as input, we now only need to model the function $h: \Omega \to \R^a$, which takes {\bf a single point $x\in S$} as input. 
For simplicity, assume that the input metric space $(\Omega, d_\Omega)$ is a compact set in some Euclidean space $\R^d$. 
Note that in contrast to Lemma \ref{thm:deep-set-final-step}, each $h$, $\phi$ and $\rho$ is simply a continuous function, and there is no further group invariance requirement. 
Furthermore, all the dimensions of the domain and range of these functions are bounded values that depend only on the covering number of $\Omega$, the target additive error $\epsilon$, and independent to the maximum size $N$ of input points. 
We can use multilayer perceptrons (MLPs) $h_{\theta_1}$, $\phi_{\theta_2}$, and $\rho_{\theta_3}$ in place of $h$, $\phi$ and $\rho$ to approximate the desired function.
Formally, we define the following family of neural networks: 
\begin{equation}
\label{eq:network-definition}
    \Network(A, B) = \rho_{\theta_3} \Big( \phi_{\theta_2} \Big( \sum_{(x, w_x) \in A} w_xh_{\theta_1}(x)\Big) + \phi_{\theta_2} \Big( \sum_{(x, w_x) \in B} w_xh_{\theta_1}(x) \Big) \Big).
\end{equation}
In practice, we consider the input to the the neural network $h_{\theta_1}$ to be a point $x \in \Omega$ along with its weight $w_x$. As per the discussions above, functions represented by $\Network$ can universally approximate \SFGI{} product functions on the space of point sets.
See Figure \ref{fig:product-network} for an illustration of our universal architecture for approximating product functions on point sets. 
As $p$-Wasserstein distances, $\mathrm{W}_p: \mathcal{X} \times \mathcal{X} \to \R$, are uniformly continuous with respect to the underlying metric $\mathrm{W}_p$, we can apply our framework for the problem of approximating $p$-Wasserstein.

Importantly, the number of parameters in $\Network$ does not depend on the maximum size of the point set but rather only on the $\epsilon$ additive error and by extension, the covering number of the original metric space.
This is because the encoding function for our sketch is defined as the summation of single points into $\R^a$ where $a$ is independent of the size of the input set. Contrast this result with latent dimension in the universality statement of DeepSets (cf. Theorem \ref{thm:sum-decomposition}), which is dependent on the input point set size. 
Note that in general, the model complexities of the MLPs $\rho_{\theta_3}$, $\phi_{\theta_2}$, and $h_{\theta_1}$ depend on the dimensions of the domain and co-domain of each function they approximate ($\rho$, $\phi$ and $h$) and the desired approximation error $\epsilon$.
We assume that MLPs are Lipschitz continuous. 
In our case, $\phi_{\theta_2}$ operates on the 
sum of $h_{\theta_1}(x)$s for all $N$ number of input points in $A$ (or in $B$). In general, the error made in $h$s may accumulate $N$ times, which causes the precision we must achieve for each individual $h_{\theta_1}(x)$ (compared to $h(x)$) to be $\Theta(\epsilon / N)$. This would have caused the model complexity of $h_{\theta_1}$ to depend on $N$.  
Fortunately, this is not the case for our encoding function $h$. In particular, because our encoding function can be written in the specific form of a normalized sum of individual points in the set $S$;  i.e, $\phi_{\theta_2}$ operates on $\sum_{(x,w_x)\in S} w_x h_{\theta_1}(x)$ with $\sum_x w_x = 1$, the error accumulated by the normalized sum will be less than the maximum error from any single point $h(x)$ for $x\in S$. Thus, as both the error for each MLP and the dimension of the domain and co-domain of each approximated function $(\rho, \phi, h)$ do not depend on the size of the input point set $N$, we get that $\rho_{\theta_3}$, $\phi_{\theta_2}$ and $h_{\theta_1}$ each have model complexity independent of the size of the input point set.
In short, because the encoding and decoding functions of our sketch can be approximated with neural networks with model complexity \textit{independent} of input point set size, we are able to achieve a low model complexity neural network which can approximate Wasserstein distance arbitrarily well. 
Note that in general, we are not guaranteed to be able to find a sketch which can be approximated by neural networks which are independent of input size.
Finally, this neural network architecture can also be applied to approximating Hausdorff distance. More details regarding Hausdorff distance approximation are available in Appendix \ref{appendix:hausdorff}.
We summarize the above results in the following corollary. 

\begin{corollary}\label{cor:goodNN}
There is a network in the form of $\Network$ which can approximate the $p$-Wasserstein distance between point sets to within an additive $\epsilon$-error. The number of parameters in this network depends on $\epsilon$ and the covering number of $\Omega$ and not on the size of each point set.

\begin{figure}
    \centering
    \includegraphics[scale=0.50]{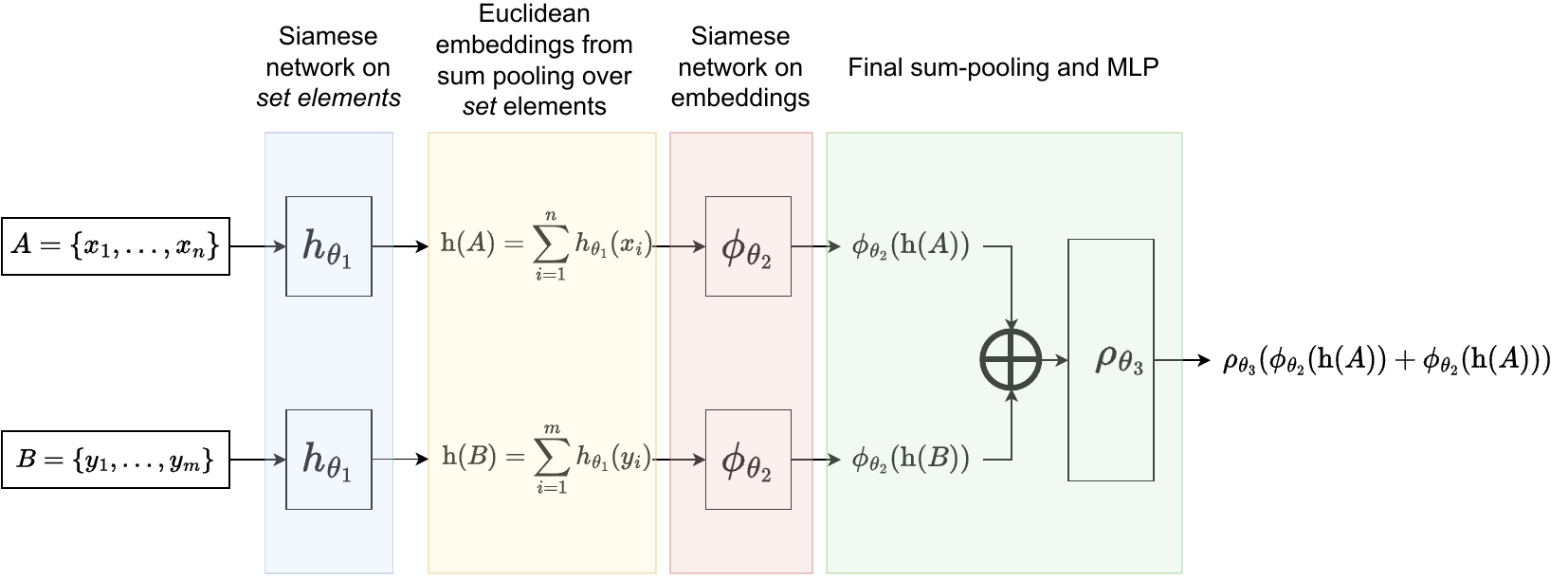}
    \caption{{\small Visually representing a neural network which can universally approximate uniformly continuous \SFGI{} product functions over pairs of point sets.} }
    \label{fig:product-network}
\end{figure}
 
Additionally, if we replace the sum-pooling with $\max$ in $\Network$, there is a network of such a form where we can also approximate Hausdorff distance between point sets to additive $\epsilon$ accuracy.
\end{corollary}

\paragraph{Exponential dependence on dimension}
Although the model complexity of $\mathcal{N}_{\mathrm{ProductNet}}$ is independent of the size of the input point set, it depends on the covering number of $\Omega$ which, in turn, can have an exponential dependence on the dimension of $\Omega$. 
In short, this means that the model complexity of $\mathcal{N}_{\mathrm{ProductNet}}$ has an exponential dependence on the dimension of $\Omega$. 
However, in practice, many machine learning tasks (e.g. 3D point processing) involve large point sets sampled from low-dimensional space $(d = 3)$. 
Furthermore, in general, the covering number of $\Omega$ will depend on the \textit{intrinsic dimension} of $\Omega$ rather than the \textit{ambient dimension. }
For instance, if the input point sets are sampled from a hidden manifold of dimension $d'$ (where $d'$ which is much lower than the ambient dimension $d$), then the covering number would depend only on $d'$ and the curvature bound of the manifold. 
In many modern machine learning applications, it is often assumed that the data is sampled from a hidden space of low dimension (the manifold hypothesis) although the ambient dimension might be very high. 

\paragraph{Using max-pooling instead of sum-pooling.} Observe that in the final step of combining the Euclidean outputs for two point sets $A, B \in \mathcal{X}$, $\sum_{(x, w_x)\in A}w_x h_{\theta_1}(x) + \sum_{(x, w_x) \in B}w_x h_{\theta_1}(x)$, we use the sum of these two components (as in a DeepSet architecture) : $\sum_{(x, w_x) \in A}w_x h_{\theta_1}(x)$ and $\sum_{(x, w_x) \in B}w_x h_{\theta_1}(x)$, to ensure the symmetric condition of \SFGI{} product functions. One could replace this final sum with a final max such as in PointNet. However, to show that PointNets are able to universally approximate continuous functions $F: K \to \R$ where $K \subseteq \R^{a \times 2}$ is compact, we need to use a $\delta$-net for $K$ which will also serve as the intermediate dimension for $\phi$. As $K \subseteq [0, N]^{a \times 2}$ in our case (where $N$ is the maximum cardinality for a point set), the intermediate dimension for a max-pooling invariant architecture at the end (i.e. PointNet) now depends on the maximum size of input point sets. 

\begin{table}[t]
    \centering
    \caption{Mean relative error between approximations and ground truth Wasserstein distance between point sets. The top row for each dataset shows the approximation quality for point sets with input sizes that were seen at training time; while he bottom row shows the approximation quality for point sets with input sizes that were not seen at training time. Note that $\mathcal{N}_{\mathrm{ProductNet}}$is our model.}
    \small
    \begin{tabular}{lccccc}
          \toprule
          Dataset & Input size & $\Network$ & WPCE & $\mathcal{N}_{\mathrm{SDeepSets}}$ & Sinkhorn  \\
          \midrule
          \multirow{2}{*}{noisy-sphere-3} & [100, 300] & \textcolor{red}{\textbf{0.046 $\pm$ 0.043}} & 0.341 $\pm$ 0.202 & 0.362 $\pm$ 0.241 & 0.187 $\pm$ 0.232\\
         & [300, 500] & \textcolor{red}{\textbf{0.158 $\pm$ 0.198}} & 0.356 $\pm$ 0.286 & 0.608 $\pm$ 0.431 & 0.241 $\pm$ 0.325\\
         \midrule
         \multirow{2}{*}{noisy-sphere-6} & [100, 300] & \textcolor{red}{\textbf{0.015 $\pm$ 0.014}} & 0.269 $\pm$ 0.285 & 0.291 $\pm$ 0.316 & 0.137 $\pm$ 0.122\\
         & [300, 500] & \textcolor{red}{\textbf{0.049 $\pm$ 0.054}} & 0.423 $\pm$ 0.408 & 0.508 $\pm$ 0.473 & 0.198 $\pm$ 0.181\\
         \midrule
         \multirow{2}{*}{uniform} & 256 & \textcolor{red}{\textbf{0.097 $\pm$ 0.073}} & 0.120 $\pm$ 0.103 & 0.123 $\pm$ 0.092 & 0.073 $\pm$ 0.009\\
         & [200, 300] & \textcolor{red}{\textbf{0.131 $\pm$ 0.096}} & 1.712 $\pm$ 0.869 & 0.917 $\pm$ 0.869 & 0.064 $\pm$ 0.008\\
         \midrule
         \multirow{2}{*}{ModelNet-small} & [20, 200] & 0.084 $\pm$ 0.077 & \textcolor{red}{\textbf{0.077 $\pm$ 0.075}} & 0.105 $\pm$ 0.096 & 0.101 $\pm$ 0.032\\
         & [300, 500] & \textcolor{red}{\textbf{0.111 $\pm$ 0.086}} & 0.241 $\pm$ 0.198 & 0.261 $\pm$ 0.245 & 0.193 $\pm$ 0.155\\
         \midrule
         \multirow{2}{*}{ModelNet-large} & 2048 & \textcolor{red}{\textbf{0.140 $\pm$ 0.206}} & 0.159 $\pm$ 0.141 & 0.166 $\pm$ 0.129 & 0.148 $\pm$ 0.048\\
         & [1800, 2000] & \textcolor{red}{\textbf{0.162 $\pm$ 0.228}} &  0.392 $\pm$ 0.378 & 0.470 $\pm$ 0.628 & 0.188 $\pm$ 0.088\\
         \midrule
         \multirow{2}{*}{RNAseq} & [20, 200] & \textbf{\textcolor{red}{0.012 $\pm$ 0.010}} & 0.477 $\pm$ 0.281 & 0.482 $\pm$ 0.291 & 0.040 $\pm$ 0.009 \\
         & [300, 500] & 0.292 $\pm$ 0.041 &0.583 $\pm$ 0.309  & 0.575 $\pm$ 0.302 &  \textbf{\textcolor{red}{0.048 $\pm$ 0.006}}\\
         \bottomrule
    \end{tabular}
    \label{tab:Wasserstein-error}
\end{table}

\begin{table}[t]
    \centering
    \caption{Training time and number of epochs needed for convergence for best model} 
    \begin{tabular}{llccc}
          \toprule
          Dataset &  & $\Network$ & WPCE & $\mathcal{N}_{\mathrm{SDeepSets}}$ \\
          \midrule
          \multirow{2}{*}{noisy-sphere-3} & Time & 6min & 1hr 46min & 9min \\
         & Epochs & 20 & 100 & 100 \\
         \midrule
         \multirow{2}{*}{noisy-sphere-6} & Time & 12min & 4hr 6min & 1hr 38min\\
         & Epochs & 20 & 100 & 100 \\
         \midrule
         \multirow{2}{*}{uniform} & Time & 7min & 3hr 36min & 1hr 27min\\
         & Epochs & 23 & 100 & 100 \\
         \midrule
         \multirow{2}{*}{ModelNet-small} & Time & 7min & 1hr 23min & 12 min  \\
         & Epochs & 20 & 100 & 100\\
         \midrule
         \multirow{2}{*}{ModelNet-large} & Time & 8min & 3hr 5min & 40min \\
         & Epochs & 20 & 100 & 100\\
         \midrule
         \multirow{2}{*}{RNAseq(2k)} & Time & 15min & 14hr 26min & 3hr 01min\\
         & Epochs & 73 & 100 & 100\\
         \bottomrule
    \end{tabular}
    \label{tab:training-time}
\end{table}

\section{Experimental results}
\label{section:experiments}
We evaluate the accuracy of the 1-Wasserstein distance approximations of our proposed neural network architecture, {\bf $\Network$}, against two different baseline architectures: (1) a Siamese autoencoder known as the Wasserstein Point Cloud Embedding network (WPCE) \citep{kawano2020learning} (previously introduced at the end of Section \ref{section:introduction} and is a SOTA method of neural approximation of Wasserstein distance) and (2) a Siamese DeepSets, denoted as $\mathcal{N}_{\mathrm{SDeepSets}}$, which is a single DeepSets model which maps both point sets to a Euclidean space and approximates the 1-Wasserstein distance as the $\ell_2$ norm between output of each point set. As Siamese networks are widely employed for metric learning, $\mathcal{N}_{\mathrm{SDeepSets}}$ model is a natural baseline for comparison again $\Network$. We additionally test our neural network approximations against the Sinkhorn distance where the regularization parameter was set to $\epsilon=0.1$. For each model, we record results for the best model according to a hyperparameter search with respect to the parameters of each model. Finally, we use the ModelNet40 \citep{wu20153d} dataset which consists of point clouds in $\R^3$ and a gene expression dataset (RNAseq) which consists of 4360 cells each represented by 2000 genes (i.e. 4360 points in $\R^{2000}$) as well as three synthetic datasets: (1) uniform, where point sets are in $\R^2$, (2) noisy-sphere-3, where point sets are in $\R^3$, (3) noisy-sphere-6, where point sets are in $\R^6$. The RNAseq dataset is publicly available courtesy of the Allen institute \citep{yao2021taxonomy}. Additional details and experiments approximating the 2-Wasserstein distance are available in Appendix \ref{appendix:experiments}.

\paragraph{Approximating Wasserstein distances.}
Our results comparing 1-Wasserstein distance approximations are summarized in Table \ref{tab:Wasserstein-error}. Additionally, see Table \ref{tab:computation-time} for a summary of time needed for training. For most datasets, $\Network$ produces more accurate approximations of Wasserstein distances for both input point sets seen at training time and for input point sets unseen at training time. For the high dimensional RNAseq dataset, our approximation remains accurate in comparison with other methods, including the standard Sinkhorn approximation for input point sets seen at training time. The only exception is ModelNet-small, where the $\Network$ approximation error is slightly larger than WPCE for input point set sizes using during training (top row for each dataset in Table \ref{tab:Wasserstein-error}). However, for point sets where the input sizes were not used during training (bottom row for each dataset in Table \ref{tab:Wasserstein-error}), $\Network$ showed siginificantly lower error than all other methods including WPCE. These results along with a more detailed plot in Figure \ref{fig:generalization_graph} in Appendix \ref{appendix:experiments} indicate that $\Network$ generalizes better than WPCE to point sets of input sizes unseen at training time. Also, see Appendix \ref{appendix:experiments} for additional discussion about generalization. Furthermore, one major advantage of $\Network$ over WPCE is the dramatically reduced time needed for training (cf. Table \ref{tab:training-time}). 
This substantial difference in training time is due to WPCE's usage of the Sinkhorn reconstruction loss as the $O(n^2)$ computation time for the Sinkhorn distance can be prohibitively expensive as input point set sizes grow. 
Thus, our results indicate that, compared to WPCE, $\Network$ can reduce training time while still achieving comparable or better quality approximations of Wasserstein distance. 
Using our $\Network$, we can produce high quality approximations of 1-Wasserstein distance while avoiding the extra cost associated with using an autoencoder architecture and Sinkhorn regularization. Finally, all models produce much faster approximations than the Sinkhorn distance (see 
Tables \ref{tab:computation-time} and \ref{tab:w2-computation-time} in Appendix \ref{appendix:experiments}). In summary, as compared to WPCE, our model is more accurate in approximating both 1-Wasserstein distance, generalizes better to larger input point set sizes, and is more efficient in terms of training time.

\section{Concluding Remarks}
Our work presents a general neural network framework for approximating \SFGI{} functions which can be combined with geometric sketching ideas into a specific and efficient neural network for approximating $p$-Wasserstein distances. 
We intend to utilize $\Network$ as an {\bf accurate, efficient, and differentiable} approximation for Wasserstein distance in downstream machine learning tasks where Wasserstein distance is employed, such as loss functions for aligning single cell multi-omics data \citep{demetci2020gromov} or compressing energy profiles in high energy particle colliders \citep{di2021reconfigurable, shenoyemd}. 
Beyond Wasserstein distance, we will look to apply our framework to a wide array of geometric problems that can be considered \SFGI{} functions and are desireable to approximate via neural networks.
For instance, consider the problems of computing the optimal Wasserstein distance under rigid transformation or the Gromov-Wasserstein distance, which both can be represented as an \SFGI{} function where the factor-wise group invariances include both permutation and rigid transformations. 
Then our sketch must be invariant to both permutations and orthogonal group actions on the left.
It remains to be seen if there is a neural network architecture which can approximate such an \SFGI{} function to within an arbitrary additive $\epsilon$-error where the complexity does not depend on the maximum size of the input set.

\section{Acknowledgements}
The authors thank anonymous reviewers for their helpful comments.
Furthermore, the authors thank Rohan Gala and the Allen Institute for generously providing the RNAseq dataset. 
Finally, Samantha Chen would like to thank Puoya Tabaghi for many helpful discussions about permutation invariant neural networks and Tristan Brug\`ere for his implementation of the Sinkhorn algorithm. 
This research is supported by National Science Foundation (NSF) grants CCF-2112665 and CCF-2217058 and National Institutes of Health (NIH) grant RF1 MH125317.

\bibliography{main.bib}

\newpage

\appendix

\section{Approximating Hausdorff distance}
\label{appendix:hausdorff}
Suppose the underlying metric for $\mathcal{X}$ (the space of weighted point sets) is Hausdorff distance, $\mathrm{d}_H$. To clarify, given $S_1 = \{(x_i, w_i) : x_i \in \R^d, \sum_{i = 1}^n w_i\}$ and $S_2 = \{(y_i, w_i': y_i \in \R^d, \sum_{i = 1}^m w_i'\}$, 
\begin{equation*}
    \mathrm{d}_H(S_1, S_2) = \mathrm{d}_H(\{x_1, \dots, x_n\}, \{y_1, \dots, y_m\}).
\end{equation*}
To get an encoding/decoding for $\mathcal{X}$ equipped with Hausdorff distance using max-pooling, we will use the construction given in the proof of universality in \citep{qi2017pointnet}, which we will briefly recap here.
As in Lemma \ref{lemma:point-sketch}, we consider a $\delta$-net for our original space $X$, $\{y_1, \dots, y_a\}$. For each $y_i$, we will again let $h_i(x) = e^{-\mathrm{d}_X(x, B_{\delta}(y_i))}$ and define $h: X \to \R^a$ as $h(x) = [h_1(x), \dots, h_a(x)]$. However, instead of using sum-pooling to define $\mathbf{h}: \mathcal{X} \to \R^a$, we will use max-pooling as
\begin{equation*}
    \mathrm{h}(S) = \max_{x \in S} h(x)
\end{equation*}
where $\max$ is element-wise max.
Then the decoding function $g: \R^a \to \mathcal{X}$ is defined as 
\begin{equation*}
    g(v) = \{ (y_i, \frac{v_i}{\|v\|_1}) : v_i \geq 1\}
\end{equation*}
Note that $v_i \geq 1$ when there is an $x_i \in S$ such that $x_i \in B_{\delta}(y_i)$. So for each $x_i \in S$ there is $y_i \in g \circ h(S)$ such that $\mathrm{d}_X(x_i, y_i) < \delta$. Thus, $\mathrm{d}_H(S, g\circ h(S)) < \delta$.

The final $\Network$ for approximating uniformly continuous functions for Hausdorff distance with max-pooling is
\begin{equation*}
    \Network^{\max} (A, B) = \rho_{\theta_3}\Bigg( \phi_{\theta_2}\Big(\max_{x \in A}(h_{\theta_1}(x))\Big) + \phi_{\theta_2} \Big(\max_{y \in B}(h_{\theta_1}(y))\Big)\Bigg).
\end{equation*}
Note that $\Network^{\max}$ is able to approximate Hausdorff distance to $\epsilon$-accuracy.

\section{Comparison with other 
neural network architectures for approximating Wasserstein distance} 
\label{subsec:comparisonNNs}
Through the lens of 1-Wasserstein approximation, we compare the power of $\Network$ with other natural architectures, namely standard MLPs and Siamese networks. The standard MLP and Siamese networks are the natural points of comparison for $\Network$ as (1) MLPs are the go-to choice in deep learning due to their versatility and efficacy across diverse problem domains and (2) Siamese networks have gained significant prominence in the context of metric learning, frequently employed for tasks aimed at learning effective similarity metrics between pairs of inputs\citep{hadsell2006dimensionality, bromley1993signature, koch2015siamese}
However, we discuss the limitations of both the standard MLP and Siamese networks for the taks of learning Wasserstein distances. For this setting, we will let $(X, \mathrm{d}_X)$ be some compact Euclidean space, i.e. $X \subseteq \R^d$ is compact and $\mathrm{d}_X$ is some $\ell_p$ norm.
Then $\mathcal{X}$ be the set of weighted point sets of cardinality at most $N$ i.e. $S = \{(x_i, w_i) : w_i \in \R_+, \sum_i w_i = 1, x_i \in X\}$. 


\paragraph{Comparison with multilayer perceptrons (MLPs).} 
First, we consider a standard MLP model where the model takes a single tensor as input. 
 Note that for each $S$, there is an $\tilde{S} \in \mathcal{X}$ such that 
\begin{equation*}
   \tilde{S} = \{(x_i, w_i) : (x_i, w_i) \in S\} \cup \{(x_1, 0)_{\times (N- |S|)}\}
\end{equation*}
where $(x_1, 0)_{\times(N - |S|)}$ means that we repeat the element $(x_1, 0)$ $(N - |S|)$ times. Notice that $\mathrm{W}_p(S, \tilde{S}) = 0$ and for any $S' \in \mathcal{X}$, $\mathrm{W}_p(S, S') = \mathrm{W}_p(\tilde{S}, S')$. If $|S|=N$, $\tilde{S} = S$
To use an MLP model to approximate Wasserstein distance, we want to represent any $(S_1, S_2) \in \mathcal{X} \times \mathcal{X}$ as an element of $\R^{(d + 1) \times 2N}$ where $S_1 = \{(x_i, w_i): w_i \in \R_+, \sum_i w_i = 1, x_i \in X\}$ and $S_2 = \{(x_i', w_i'): w_i' \in \R_+, \sum_i w_i' = 1, x_i' \in X\}$. 
Informally, given any ``empty" element in $\R^{(d + 1) \times 2N}$, we will map $S_1$ to the first $N$ columns and map $S_2$ to the last $N$ columns. 
Formally, we will define the mapping $\beta: \mathcal{X} \times \mathcal{X} \to \R^{(d + 1) \times 2N}$ given $S_1$ and $S_2$ as defined above as 
\begin{equation*}
    \beta(S_1, S_2) = \beta(\tilde{S}_1, \tilde{S}) = \begin{bmatrix}
        x_1 & \cdots & x_N & x_1' & \cdots & x_N' \\
        w_1 & \cdots & w_N & w_1' & \cdots & w_N'
    \end{bmatrix}
\end{equation*}
where we abuse notation slightly and use $(x_i, w_i)$ and $(x_i', w_i')$ to describe the elements in $\tilde{S}_1$ and $\tilde{S}_2$ respectively. Given $M \in \R^{(d + 1) \times 2N}$, define $\beta^{-1}: \mathrm{Image}(\beta) \to \mathcal{X} \times \mathcal{X}$ as
\begin{equation*}
    \beta^{-1}\Bigg(\begin{bmatrix}
        x_1 & \cdots & x_N & x_1' & \cdots & x_N' \\
        w_1 & \cdots & w_N & w_1' & \cdots & w_N'
    \end{bmatrix} \Bigg) = (\{(x_i, w_i)\}, \{(x_i', w_i')\})
\end{equation*}
Now, in order to use the classical universal approximation theorem to approximate $\mathrm{W}_p$ to an $\epsilon$-approximation error, we must show that there is a continuous function $f:\mathrm{Image}(\beta) \to \R$ such that $\mathrm{W}_p(S_1, S_2) = f(\beta(S_1, S_2))$. 
\begin{lemma}
\label{lemma:mlp-wasserstein}
    Given $\beta: \mathcal{X} \times \mathcal{X} \to \R^{(d+ 1) \times 2N}$ as defined previously, there is a continuous function $f: \mathrm{Image}(\beta) \to \R$ such that for any $\epsilon > 0$, and for any $(S_1, S_2 )\in \mathcal{X} \times\mathcal{X}$, $|\mathrm{W}_p(S_1, S_2) - f(\beta(S_1, S_2))| < \epsilon$. 
\end{lemma}
\begin{proof}
    Let $\epsilon > 0$. From Corollary \ref{corollary:product-net-universality}, there is a $\delta > 0$ such that are continuous $h: X \to \R^{a(\delta)}$, $\phi: \R^{a(\delta)} \to \R^{a'}$, and $\rho: \R^{a'} \to \R$, such that for any $A, B \in \mathcal{X}$
    \begin{equation*}
        \Bigg|\mathrm{W}_p(A, B) - \rho\Big(\phi\Big(\sum_{x \in A} w_x h(x)\Big) + \phi\Big(\sum_{x \in B}w_x h(x)\Big)\Big) \Bigg| < \epsilon.
    \end{equation*}
    Let $(S_1, S_2) \in \mathcal{X} \times \mathcal{X}$ and let $M = \beta(S_1, S_2)$. Note that $\beta^{-1}(M) = (S_1, S_2)$.
    Let $M \in \mathrm{Image}(\beta)$, let $M[:, i]$ denote the $i$th column vector of $M$, let $M[:d, i]$ denote vector of the first $d$ elements in the $i$th column vector of $M$ and let $M[i, j]$ denote the $ij$th entry of $M$. We know that $M[:, i] \in \R^{d + 1}$. Since $M[:, i] \in \mathrm{Image}(\beta)$, we know that $M[:d, i] \in X$ and $M[d, i]$ corresponds to a weight.
    Furthermore, since $h$ is a continuous mapping from $X \to \R$, $M[d, i]h(M[:d, i])$ is continuous. Thus, since $\phi$ is also continuous, the function $f': \mathrm{Image}(\beta) \to \R^{a'}$
    \begin{equation*}
        \phi(\sum_{i = 1}^N M[d, i]h(M[:d, i]) ) + \phi(\sum_{i = N + 1}^{2N} M[d, i]h(M[:d, i]))
    \end{equation*}
     is continuous. Since $\rho$ is also continuous, we get that $\rho \circ f': \mathrm{Image}(\beta) \to \R$ is continuous. Thus, $f: \mathrm{Image}(\beta) \to \R$ which is $f = \rho \circ f'$ is also continuous and $|f(\beta(S_1, S_2)) - \mathrm{W}_p(S_1, S_2)| < \epsilon$ by the property from Corollary \ref{corollary:product-net-universality}
\end{proof}

Since $f$ in the above Lemma operates on fixed dimensional space, we can now use an MLP to approximate it. Then by the above Lemma and universal approximation of MLPs, there is an MLP which can approximate
$\mathrm{W}_p: \mathcal{X} \times \mathcal{X} \to \R$ to an arbitrary $\epsilon$ additive error. 
However, this approach has the following issues: 
while an MLP can approximate $\mathrm{W}_p$ to an arbitrary $\epsilon$ additive error, the model complexity of the MLP depends on not only the approximation error, but also on the maximum size $N$ of the input point set. In contrast, our neural network $\Network$ introduced in Eqn (\ref{eq:network-definition}) has model complexity which is independent of $N$; see Corollary \ref{cor:goodNN}. 
Furthermore, the function represented by the MLP (that we use to approximate $\phi$) however may not be symmetric to the two point sets, and may not guarantee permutation invariance with respect to $G \times G$. In practice, computationally expensive techniques such as data augmentation are often required in order for an unconstrained MLP to approximate a structured function such as our Wasserstein distance function, which is an \SFGI-function. 

\paragraph{Comparison with Siamese architectures.} 
As mentioned previously, one common way of learning product functions is to use a Siamese network which embeds $\mathcal{X}$ to Euclidean space and then use some $\ell_p$ norm to approximate the desired product function.
Consider the learning of Wasserstein distances again. 
To approximate $W_p(A, B)$ for two point sets $A$ and $B$. 
The Siamese network will first embed each point set to a fixed dimensional Euclidean space $\R^D$ via a function $\phi_\theta$ modeled by a neural network, and then compute $\|\phi_\theta(A) - \phi_\theta(B)\|_q$ for some $1\le q < \infty$. 
Intuitively, compared to our neural network $\Network$ introduced in Eqn (\ref{eq:network-definition}) (and recall Figure \ref{fig:product-network}), the Siamese network will replace the outer function $\rho$ by simply the $L_q$-norm of $\phi_\theta(A) - \phi_\theta(B)$. 
However, this approach intuitively requires that one can find a near-isometric embedding of the Wasserstein distance to the $L_q$ distance in a Euclidean space. 
However, there exists lower bound results on the distortion incurred when embedding Wasserstein distance to $L_p$ space. In the following section, we will review one such result on the lower bound of metric distortions when embedding the Wasserstein distance to $L_1$ space; implying that if we chose $q=1$, then in the worst case the Siamese network will incur at least as much distortion as that lower bound.

\section{Non-embeddability theorems for Wasserstein distances}
\label{appendix:euc-embeddings}

Here we summarize results pertaining to the limitations of embedding Wasserstein distance to the $L_q$ distance in a Euclidean space. Consider probability distributions over a grid of points in $\R^2$, $\{0,1, \dots, D\}^2$ equipped with the 1-Wasserstein distance, $(\mathcal{P}(\{0, 1, \dots, D\}^2), \mathrm{W}_1)$. Let $L_1$ denote the space of Lebesgue measureable functions $f: [0, 1] \to \R$ such that 
\begin{equation*}
    \|f\|_1 = \int_0^1 |f(t)|dt.
\end{equation*}
Given a mapping $F: (\mathcal{P}(\{0, 1, \dots, n\}^2), \mathrm{W}_p) \to L_1$ such that for any $\mu, \rho \in (\mathcal{P}(\{0, 1, \dots, D\}^2), \mathrm{W}_p)$, 
\begin{equation}
    \mathrm{W}_p(\mu, \rho) \leq \|F(\mu) - F(\rho)\| \leq L \cdot \mathrm{W}_p(\mu, \rho)
\end{equation}
the \textit{distortion} is the value of $L$. 
\begin{theorem}[\citep{naor2007planar}]
    Any embedding $(\mathcal{P}(\{0,1, \dots, n\}^2, \mathrm{W}_1) \to L_1$ must incur distortion at least $\Omega(\sqrt{\log D})$
\end{theorem}
For $(\mathcal{P}(\{0, 1, \dots, D\}^d), \mathrm{W}_1)$ where $d \geq 2$, $\mathcal{P}(\{0, 1, \dots, D\}^d$ contains $\mathcal{P}(\{0, 1, \dots, D\}^2)$ so the lower bound $O(\sqrt{D})$ still applies for $L_1$ embeddings from $\mathcal{P}(\{0, 1, \dots,D\}^d)$, $d > 2$.

From the above results, we can see that for any Siamese architecture (even in the simple case of finite point sets in $\R^2$ and $\R^3$), we are unable to approximate the Wassserstein distances via any $L_1$ embedding. 

\begin{corollary}
    Given a neural network $\mathcal{N}_{\mathrm{Siamese}}: \mathcal{X} \to \R^d$ where $\mathcal{X}$ are weighted point sets over $\{0, 1, \dots, D\}^2$, $\|\mathcal{N}_{\mathrm{Siamese}}(\mu) - \mathcal{N}_{\mathrm{Siamese}}(\nu)\|_1$ incurs distortion at least $\Omega(\sqrt{\log D})$.
\end{corollary}
Note that if we consider our input for a Siamese architecture to be finite point sets over $\{0,1, \dots, D\}^2$, we allow multisets so the input set size is not bounded by $D$. 

\section{Experimental details}
\label{appendix:experiments}
\subsection{Baseline models for comparison with $\Network$}
Here, we will detail two baseline neural networks in our experiments.
\paragraph{Wasserstein point cloud embedding (WPCE) network} First defined by \citep{kawano2020learning}, the WPCE network is an Siamese autoencoder architecture. It consists of an encoder network $\mathcal{N}_{\mathrm{encoder}}$ and a decoder network $\mathcal{N}_{\mathrm{decoder}}$. WPCE takes as input two point sets $P, Q \subseteq \R^d$.  $\mathcal{N}_{\mathrm{encoder}}$ is a permutation invariant neural network - which we chose to be DeepSets. In other words,
\begin{equation*}
    \mathcal{N}_{\mathrm{encoder}}(P) = \phi_{\theta_2}\Big( \sum_{x \in P} h_{\theta_1}(x) \Big).
\end{equation*}
Note that one may also choose PointNet to be $\mathcal{N}_{\mathrm{encoder}}$. However, in our experiments, we did not see a large difference in approximation quality between using PointNet and DeepSets (where the difference between the two amounts to using a sum vs. max aggregator). For sake of consistency with $\Network$, we chose to use a sum aggregator for $\mathcal{N}_{\mathrm{encoder}}$ (DeepSets). The decoder network, $\mathcal{N}_{\mathrm{decoder}}$, is a fully connected neural network which outputs a fixed-size point set in $\R^d$. WPCE then uses a Sinkhorn reconstruction loss term to regularize the embedding produced by the encoder. Thus, given a set of paired input point sets and their Wasserstein distance $\{(P_i, Q_i, \mathrm{W}_1(P_i, Q_i)) : i \in [N]\}$ the loss which we optimize over for WPCE is 
\begin{multline}
\label{eq:WPCE-loss}
    \mathcal{L}(P, Q) = \frac{1}{N} \sum \Big(\|\mathcal{N}_{\mathrm{encoder}}(P_i) - \mathcal{N}_{\mathrm{decoder}}(Q_i)\|_2 -  W_1(P_i, Q_i)\Big)^2 \\+ \lambda \Big(\frac{1}{N} \sum S_{\epsilon}(\mathcal{N}_{\mathrm{decoder}}(\mathcal{N}_{\mathrm{encoder}}(P_i)) + \frac{1}{N} \sum S_{\epsilon}(\mathcal{N}_{\mathrm{decoder}}(\mathcal{N}_{\mathrm{encoder}}(Q_i)) \Big)
\end{multline}
where $\lambda$ is some constant in $\R$ that controls the balance between the two loss terms and $S_{\epsilon}$ is the Sinkhorn divergence between $P_i$ and $Q_i$. Note that $\epsilon$ in $S_{\epsilon}$ refers to the regularization parameter for the Sinkhorn divergence. For our experiments, we chose $\lambda = 0.1$ and $\epsilon = 0.1$. Furthermore, for each dataset used in our experiments, we set used a range of different sizes for the fixed-size output point set. This parameter is summarized per dataset in  Table \ref{tab:WPCE-decoding}.
\begin{table}[t]
    \centering
    \caption{Size of the output point set from the WPCE decoder for each dataset. }
    \begin{tabular}{lcccccc}
        \toprule
             & noisy-sphere-2 & noisy-sphere-6 & uniform & ModelNet-small & ModelNet-large & RNAseq\\
        \midrule
        Size & 200 & 200 & 256 & 100 & 2048 &100\\
        \bottomrule
    \end{tabular}
    \label{tab:WPCE-decoding}
\end{table}

\paragraph{Siamese DeepSets.} The baseline Siamese DeepSets model, $\mathcal{N}_{\mathrm{SDeepSets}}(\cdot, \cdot)$, consists of a single DeepSets model which maps two input point sets to some Euclidean space $\R^d$. 
The $\ell_2$-norm between the final embeddings is then used as the final estimate for Wasserstein distance. Formally, let $\mathcal{N}_{\mathrm{DeepSets}}(P) = \phi_{\theta_2}\Big(\sum_{x \in P} h_{\theta_1}(x)\Big)$, where $\phi_{\theta_2}$ and $h_{\theta_1}$ are both MLPs, be the DeepSets model. 
Then given two point sets $P, Q$, the final approximation of Wasserstein distance given by $\mathcal{N}_{\mathrm{SDeepSets}}(P, Q)$ is 
\begin{equation*}
    \mathcal{N}_{\mathrm{SDeepSets}}(P, Q) = \|\mathcal{N}_{\mathrm{DeepSets}}(P) - \mathcal{N}_{\mathrm{DeepSets}}(Q)\|_2.
\end{equation*}

\begin{table}[t]
    \centering
    \caption{Summary of details for each dataset. Note that min and max refer to the minimum and maximum number of points per input point set. }
    \begin{tabular}{lcccccc}
        \toprule
                      & dim & min & max & training pairs & val pairs \\
        \midrule
        noisy-spheres-2 & 3 & 100 & 300 & 2000 & 200 \\
        noisy-spheres-6 & 6 & 100 & 300 & 3600 & 400 \\
        uniform & 2 & 256 & 256 & 3000 & 300  \\
        ModelNet-small & 3 & 20 & 200 & 3000 & 300 \\
        ModelNet-large & 3 & 2048 & 2048 & 4000 & 400 \\
        RNAseq & 2000 & 20 & 200 &3000 & 300\\
        \bottomrule
    \end{tabular}
    \label{tab:dataset-details}
\end{table}

\subsection{Training and implementation details}
\paragraph{Datasets.} We used several synthetic datasets as well as the ModelNet40 point cloud dataset. Furthermore, we used two different types of synthetic datasets. We construct the `noisy-sphere-$d$' dataset by sampling pairs of point clouds from four $d$-dimensional spheres centered at the origin with increasing radiuses of 0.25, 0.5, 0.75, and 1.0. For our experiments, we used `noisy-sphere-3' and `noisy-sphere-6'. Finally, the `uniform' dataset of points sets in $\R^2$ is constructed by sampling points sets from the uniform distribution on $[-4, 4] \times [-4, 4]$. The full details and names of each dataset are summarized in Table \ref{tab:dataset-details}.

\paragraph{Implementation.} We implement all neural network architectures using the PyTorch\citep{paszke2019pytorch} and GeomLoss \citep{feydy2019interpolating} libraries. The ground truth 1-Wasserstein distances and Sinkhorn approximations were computed using Python Optimal Transport (POT) \citep{flamary2021pot}. Note that for large point sets and for higher dimensional datasets, there is often a high degree of numerical instability in the POT implementation of the Sinkhorn algorithm. In these cases (ModelNet-large and RNAseq) we used our own implementation of the Sinkhorn algorithm. For each model, we used the Leaky-ReLU as the non-linearity. To train each model, we set the batch size for each dataset to be 128 and the learning rate to 0.001. All models were trained on an Nvidia RTX A6000 GPU. For both $\Network$ and $\mathcal{N}_{\mathrm{SDeepSets}}$, given two input point sets, we minimize on the mean squared error between the approximation produced by the network and the true Wasserstein distance. In other words, given two point sets $P, Q$, the loss for $\Network$ is defined as $\mathcal{L}_{\Network}(P, Q) = \Big(\Network(P, Q) - \mathrm{W}_1(P, Q)\Big)^2$ and the loss for Siamese DeepSets is $\mathcal{L}_{\mathcal{N}_{\mathrm{SDeepSets}}}(P, Q) = \Big(\mathcal{N}_{\mathrm{SDeepSets}}(P, Q) - \mathrm{W}_1(P, Q) \Big)^2$.
For WPCE, we train the network using the loss function defined in Eq. \ref{eq:WPCE-loss}. Note that for $\Network$, the hyperparameters are the width, depth, and output dimension for the MLPs which represent  $h_{\theta_1}$ and $\phi_{\theta_2}$ and the width and depth the MLP which represents $\rho_{\theta_3}$. For WPCE, we set the decoder to a three layer neural network with width 100 and adjusted the width, depth, and output dimension for the MLPs which represent $\phi_{\theta_2}$ and $h_{\theta_1}$ in $\mathcal{N}_{\mathrm{encoder}}$. To find the best model for each architecture, we randomly sampled hyperparameter configurations and conducted a hyperparameter search over 85 models for $\Network$ and 75 models for both WPCE and $\mathcal{N}_{\mathrm{SDeepSets}}$.

\subsection{Approximating 2-Wasserstein distance}
To further show the use of our model, we additionally approximate the 2-Wasserstein distance; see Table \ref{tab:2-Wasserstein-error} for results. The experimental set-up is the same as for 1-Wasserstein distance and we largely see the same trends as we see for 1-Wasserstein distance; that is, $\mathcal{N}_{\mathrm{ProductNet}}$ outperforms all other neural network implementations. Note that Table \ref{tab:2-Wasserstein-error} shows that the Sinkhorn approximation with $\epsilon = 0.01$ is more accurate than $\mathcal{N}_{\mathrm{ProductNet}}$. However, as the $\epsilon$ parameter for the Sinkhorn approximation decreases, the computation time increases. In Table \ref{tab:w2-computation-time}, we show that the Sinkhorn approximation is already much slower than $\mathcal{N}_{\mathrm{ProductNet}}$ at $\epsilon = 0.1$ while also having a less accurate approximation. The Sinkhorn approximation with $\epsilon = 0.01$ (reported in Table \ref{tab:2-Wasserstein-error} is slower the Sinkhorn approximation with $\epsilon = 0.1$ and additionally, is also much slower than $\mathcal{N}_{\mathrm{ProductNet}}$ at inference time. 

\subsection{Generalization to large point sets}
In addition to the results reported in Tables \ref{tab:Wasserstein-error} and \ref{tab:2-Wasserstein-error}, which record the average relative error for point sets of unseen sizes during training, we have also included several plots in Figures \ref{fig:generalization_graph} and \ref{fig:generalization-w2} which demonstrate how the error for each approximation method changes as the input point set size increases for both 1-Wasserstein distance and 2-Wasserstein distance. Observe that for ModelNet-small, noisy-sphere-3, and noisy-sphere-6, the error for $\Network$ exhibits a significantly slower increase as compared to WPCE. 
The rate at which the error increases is not as evident for the uniform and ModelNet-large datasets. 
It is worth mentioning that we trained the model on fixed-size input for both of these datasets. 
It is possible that training with fixed-size input leads a rapid deterioration in approximation quality for WPCE and $\mathcal{N}_{\mathrm{SDeepSets}}$ when dealing with point sets of sizes not seen at training time. 
Furthermore, consider that WPCE may be especially sensitive to differences in input sizes at testing time as training WPCE depends on minimizing the Wasserstein difference between the input point set and a fixed-size decoder point set which may cause the model to be overly specialized to point sets of a fixed input size. 
This observation could provide an explanation for the observed plots in both the ModelNet-large and uniform cases.
Finally, as predicted in our theoretical analysis, the performance of the model degrades for higher dimensional datasets i.e. the RNAseq dataset. 

\begin{table}[t]
    \centering
    \caption{Time for 500 1-Wasserstein distance computations in seconds. Note that we chose input point set size to be the maximum possible point set size that we trained on. Additionally, the Sinkhorn distance reported uses $\epsilon = 0.1$ as the regularization parameter. Note that as $\epsilon$ decreases, the error incurred by  the Sinkhorn approximation will decrease but the computation time will also increases. Here, when $\epsilon = 0.1$, the Sinkhorn approximation is already much slower than the neural network approximations while being less accurate.}
    
    \begin{tabular}{lcccccc}
          \toprule
           & & \multicolumn{3}{c}{Models}\\
    \cmidrule(r){3-5}
          Dataset & Input size & $\Network$ & WPCE & $\mathcal{N}_{\mathrm{SDeepSets}}$ & Sinkhorn  & Ground truth \\
          \midrule
          noisy-sphere-3       & 300  & 1.050 & 0.676 & 0.4904 & 2.591 & 2.813\\
          noisy-sphere-6     & 300  & 0.752 & 0.491 & 0.503 & 1.986 & 6.6770 \\
          uniform             & 256  & 0.155 & 0.184 & 0.137 & 15.113 & 1.018\\
          ModelNet-small            & 200  & 0.330 & 0.330 & 0.191 & 2.074 & 1.615 \\
          ModelNet-large      & 2048 & 1.174 & 1.571 & 0.612 & 239.448 &  254.947 \\
          RNAseq & 200 & 1.128 & 0.856 & 0.792   & 92.153 & 105.908 \\
         \bottomrule
    \end{tabular}
    \label{tab:computation-time}
\end{table}

\begin{table}[]
    \centering
    \caption{Comparison of the mean relative error versus overall computation time for 500 approximations of 1-Wasserstein distance for $\Network$ and the Sinkhorn distance. Note that the input point set size is the same as in Table 3 for each dataset. The parameter $\epsilon$ controls the accuracy of the Sinkhorn approximation with lower $\epsilon$ corresponding to a more accurate approximation once the Sinkhorn algorithm converges. However, notice that in some cases, the Sinkhorn algorithm with $\epsilon = 0.01$ has a higher relative error than the Sinkhorn algorithm with $\epsilon = 0.1$ as the algorithm fails to converge within a reasonable number of iterations (1000).}
    \begin{tabular}{ccccc}
    \hline
        Dataset & &$\Network$ & Sinkhorn ($\epsilon =0.10$) & Sinkhorn ($\epsilon=0.01$) \\
        \hline
         \multirow{2}{*}{{\small ModelNet-small}} & Error & 0.084 $\pm$ 0.077 & 0.187 $\pm$ 0.232 & 0.011 $\pm$ 0.003\\
         & Time (s) & 0.330 & 2.074 & 104.712 \\
         \hline
         \multirow{2}{*}{{\small ModelNet-large}} & Error & 0.140 $\pm$ 0.206 & 0.148 $\pm$ 0.048 & 0.026 $\pm$ 0.008\\
         & Time (s) & 1.174 & 239.448 & 1930.852 \\
         \hline
         \multirow{2}{*}{{\small Uniform}} & Error & 0.097 $\pm$ 0.073 & 0.073 $\pm$ 0.009 & 0.023 $\pm$ 0.098\\
         & Time (s) & 0.155 & 15.113  &  63.028\\
         \hline
         \multirow{2}{*}{noisy-sphere-3} & Error & 0.046 $\pm$ 0.043 & 0.187 $\pm$ 0.232 & 0.162 $\pm$ 0.132\\
         & Time (s) & 1.050 & 2.591 & 214.185\\
         \hline
         \multirow{2}{*}{noisy-sphere-6} & Error & 0.015 $\pm$ 0.014 & 0.137 $\pm$ 0.122 & 0.326 $\pm$ 0.135\\
         & Time (s) & 0.752 & 1.986 & 101.763\\
         \hline
         \multirow{2}{*}{RNAseq} & Error & 0.012 $\pm$ 0.010 & 0.040 $\pm$ 0.009 & 0.035$\pm$ 0.013\\
         & Time (s) & 1.128 & 92.153 & 91.573\\
         \hline
    \end{tabular}
    \label{tab:computation-time-comparison}
\end{table}

\begin{figure}
     \centering
     \begin{subfigure}[b]{0.3\textwidth}
         \centering
         \includegraphics[width=\textwidth]{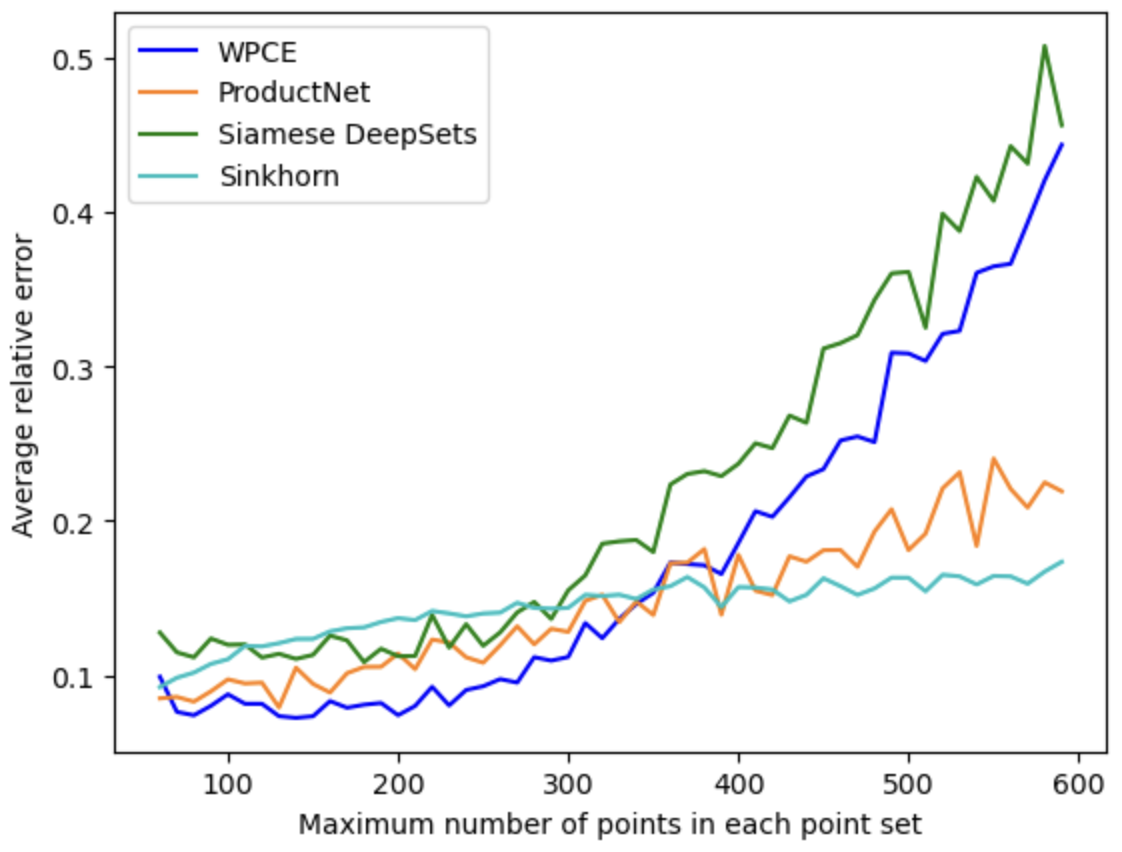}
         \caption{ModelNet-small}
         \label{fig:y equals x}
     \end{subfigure}
     \hfill
     \begin{subfigure}[b]{0.3\textwidth}
         \centering
         \includegraphics[width=\textwidth]{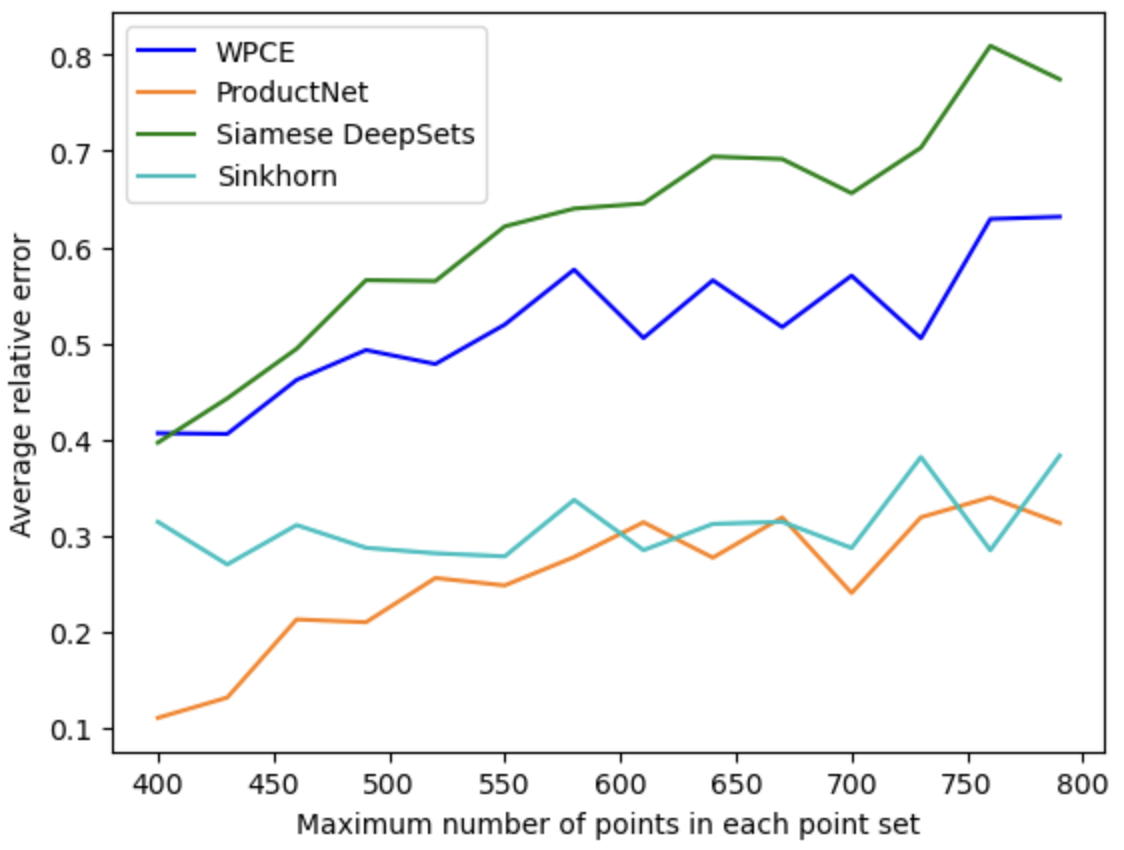}
         \caption{noisy-sphere-3.}
         \label{fig:nsphere3}
     \end{subfigure}
     \hfill
     \begin{subfigure}[b]{0.3\textwidth}
         \centering
         \includegraphics[width=\textwidth]{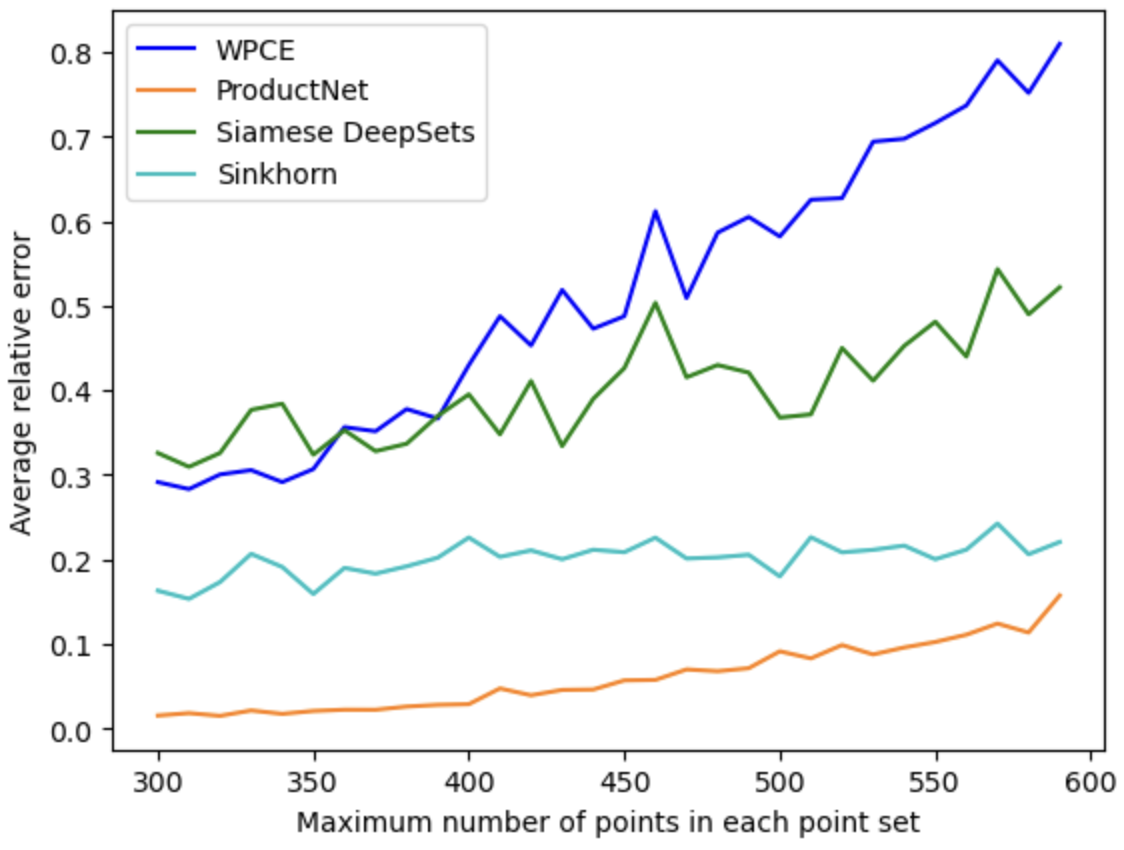}
         \caption{noisy-sphere-6.}
         \label{fig:nsphere6}
     \end{subfigure}
     \hspace*{\fill}%
     \begin{subfigure}[b]{0.3\textwidth}
         \centering
         \includegraphics[width=\textwidth]{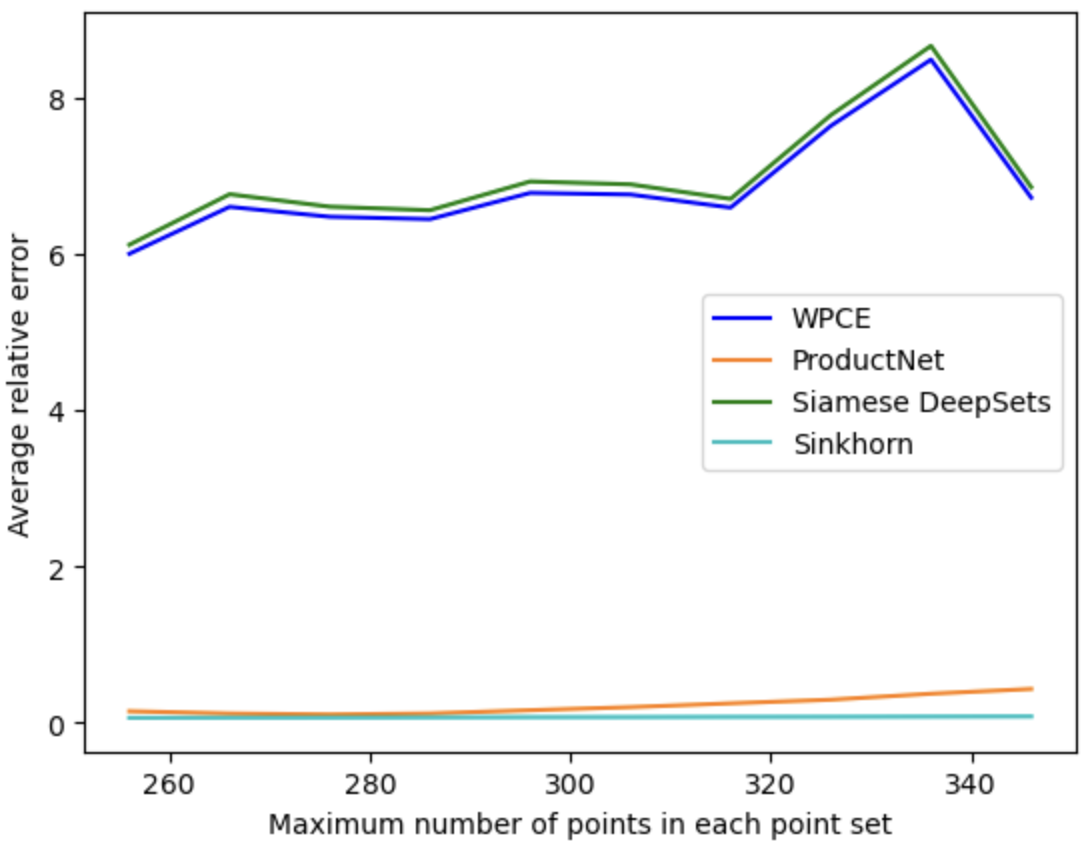}
         \caption{uniform.}
         \label{fig:uniform}
     \end{subfigure}
     \hfill
     \begin{subfigure}[b]{0.3\textwidth}
         \centering
         \includegraphics[width=\textwidth]{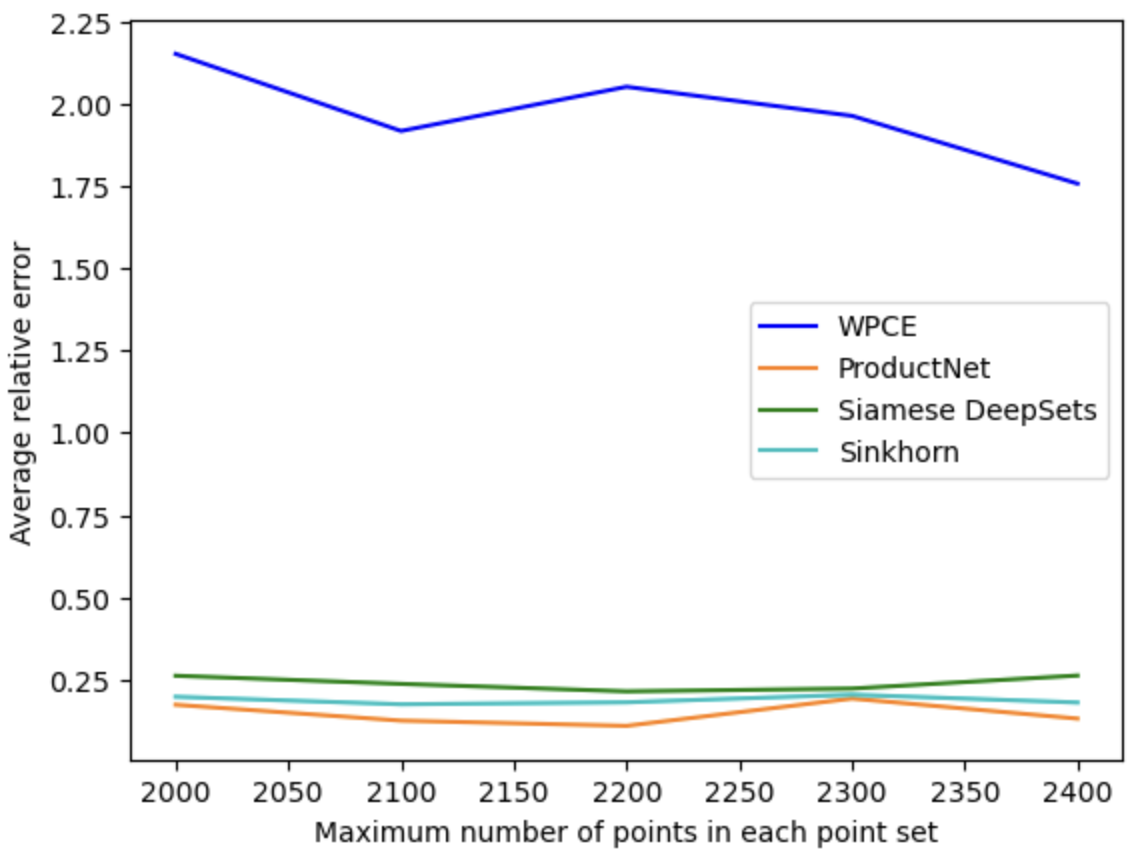}
         \caption{ModelNet-large.}
         \label{fig:modelnetlarge}
     \end{subfigure}
     \hfill
     \hfill
     \begin{subfigure}[b]{0.3\textwidth}
         \centering
         \includegraphics[width=\textwidth]{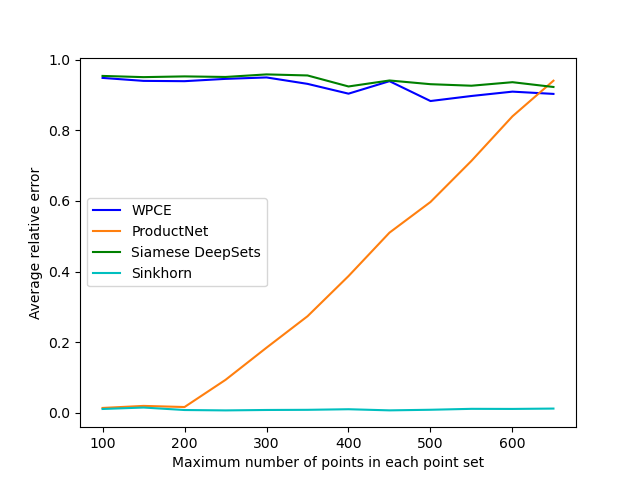}
         \caption{RNAseq.}
         \label{fig:rnaseq}
     \end{subfigure}
     \hfill
     \hspace*{\fill}%
        \caption{Average error for 1-Wasserstein approximation for each model as the maximum number of points increases. Note that the Sinkhorn approximation is the $\epsilon = 0.1$. These graphs include point set sizes not seen at training time to display how each approximation performs on unseen examples. Note that especially for ModelNet-small, noisy-sphere-3, and noisy-sphere-6, we can see that the error for $\Network$ increases at a slower rate than WPCE.}
        \label{fig:generalization-w2}
\end{figure}

\begin{table}[t]
    \centering
    \caption{Mean relative error between approximations and 2-Wasserstein distance between point sets. The top row for each dataset shows the error for point sets with input sizes that were seen at training time; while the bottom row shows the error for point sets with input sizes that were not seen at training time. Note that the Sinkhorn approximation is computed with the regularization parameter set to $\epsilon = 0.01$.}
    
    \begin{tabular}{lccccc}
          \hline
          Dataset & Input size & $\Network$ & WPCE & $\mathcal{N}_{\mathrm{SDeepSets}}$ & Sinkhorn  \\
          \hline
          \multirow{2}{*}{ noisy-sphere-3} & [100, 300] & \textcolor{red}{\textbf{0.054 $\pm$ 0.071}} & 0.291 $\pm$ 0.201 & 0.400 $\pm$ 0.336 & 0.078 $\pm$ 0.186\\
          & [400, 600] & 0.188 $\pm$ 0.197 & 0.387 $\pm$ 0.386 & 0.427 $\pm$ 0.375 & \textbf{\textcolor{red}{0.161 $\pm$ 0.311}}\\
         \hline
         \multirow{2}{*}{ noisy-sphere-6} & [100, 300] & 0.024 $\pm$ 0.010 & 0.331 $\pm$ 0.237 & 0.358 $\pm$ 0.231 & \textbf{\textcolor{red}{0.019 $\pm$ 0.057}}\\
         & [400, 600] & 0.092 $\pm$ 0.074 & 0.434 $\pm$ 0.598 & 0.623 $\pm$ 0.596 & \textbf{\textcolor{red}{0.050 $\pm$ 0.039}}\\
         \hline
         \multirow{2}{*}{uniform} & 256 &  \textbf{\textcolor{red}{0.112 $\pm$ 0.082}} & 0.221 $\pm$ 0.162  & 0.241 $\pm$ 0.171 &0.182 $\pm$ 0.044\\
         & [200, 300] & 0.175 $\pm$ 0.123 & 2.431 $\pm$ 2.162 & 4.058 $\pm$ 3.324 & 0.055 $\pm$ 0.053 \\
         \hline
         \multirow{2}{*}{ModelNet-small} & [20, 200] & 0.078 $\pm$ 0.095 & 0.178 $\pm$ 0.148 & 0.183 $\pm$ 0.148 & \textbf{\textcolor{red}{0.023 $\pm$ 0.059}}\\
         & [400, 600] & 0.163 $\pm$ 0.151 & 0.216 $\pm$ 0.252 & 0.227 $\pm$ 0.179 & \textbf{\textcolor{red}{0.034 $\pm$ 0.031}} \\
         \hline
         \multirow{2}{*}{ModelNet-large} & 2048 & 0.187 $\pm$ 0.335 & 0.281 $\pm$ 0.203 & 0.538 $\pm$ 0.298 & \textcolor{red}{\textbf{0.172 $\pm$ 0.065}} \\
         & [1800, 2000] & 0.185 $\pm$ 0.302 & 0.523 $\pm$ 0.526 & 33.086 $\pm$ 28.481 & \textcolor{red}{\textbf{0.046 $\pm$ 0.039}}\\
         \hline
         \multirow{2}{*}{RNAseq } & [20, 200] & 0.049 $\pm$ 0.029 & 0.508 $\pm$ 0.291 & 0.490 $\pm$ 0.271 & \textbf{\textcolor{red}{0.024 $\pm$ 0.009}}\\
         & [400, 600] & \textbf{\textcolor{red}{0.281 $\pm$ 0.057}} & 0.533 $\pm$ 0.300  & 0.568 $\pm$ 0.317 & 0.987 $\pm$ 0.0002\\
         \hline
    \end{tabular}
    \label{tab:2-Wasserstein-error}
\end{table}

\begin{table}[]
    \centering
    \caption{Comparison of the mean relative error versus overall computation time for 300 approximations of 2-Wasserstein distance for $\Network$ and the Sinkhorn distance. The parameter $\epsilon$ controls the accuracy of the Sinkhorn approximation with lower $\epsilon$ corresponding to a more accurate approximation.}
    \begin{tabular}{ccccc}
    \hline
        Dataset & &$\Network$ & Sinkhorn ($\epsilon =0.10$) & Sinkhorn ($\epsilon=0.01$) \\
        \hline
         \multirow{2}{*}{{\small ModelNet-small}} & Error & 0.078 $\pm$ 0.097 & 0.232 $\pm$ 0.132 & 0.019 $\pm$ 0.057\\
         & Time (s) & 0.208 & 1.165 & 9.857\\
         \hline
         \multirow{2}{*}{{\small ModelNet-large}} & Error & 0.187 $\pm$ 0.335 & 0.363 $\pm$ 0.255 & 0.172 $\pm$ 0.089\\
         & Time (s) & 2.841 & 6.079 & 36.265\\
         \hline
         \multirow{2}{*}{{\small uniform}} & Error & 0.112 $\pm$ 0.082 & 0.0303 $\pm$ 0.022 & 0.182 $\pm$ 0.044 \\
         & Time (s) & 0.712 & 16.515 & 29.312 \\
         \hline
         \multirow{2}{*}{noisy-sphere-3} & Error & 0.054 $\pm$ 0.071 & 0.225 $\pm$ 0.093 & 0.078 $\pm$ 0.186\\
         & Time (s) & 0.677 & 1.591 & 12.760\\
         \hline
         \multirow{2}{*}{noisy-sphere-6} & Error & 0.024 $\pm$ 0.010 & 0.324 $\pm$ 0.316 & 0.023 $\pm$ 0.059\\
         & Time (s) & 0.428 & 0.877 & 9.831\\
         \hline
         \multirow{2}{*}{{\small RNAseq }} & Error & 0.049 $\pm$ 0.029 & 0.031 $\pm$ 0.014 & 0.024 $\pm$ 0.009 \\
         & Time (s) & 0.716 & 47.701 & 81.016 \\
         \hline
    \end{tabular}
    \label{tab:w2-computation-time}
\end{table}

\begin{figure}
     \centering
     \begin{subfigure}[b]{0.30\textwidth}
         \centering
         \includegraphics[width=\textwidth]{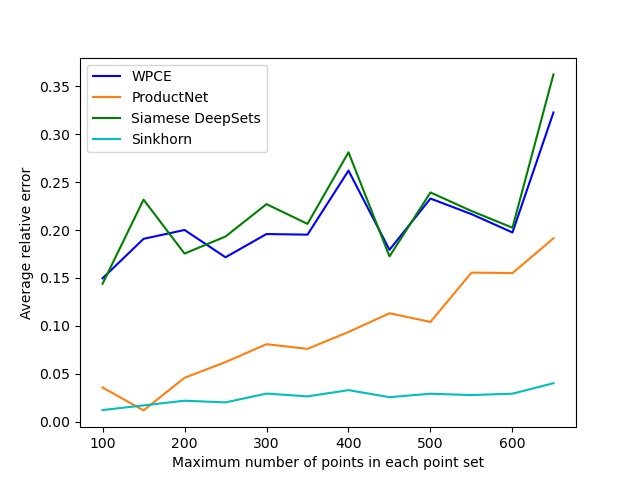}
         \caption{ModelNet-small}
     \end{subfigure}
     \hfill
     \begin{subfigure}[b]{0.30\textwidth}
         \centering
         \includegraphics[width=\textwidth]{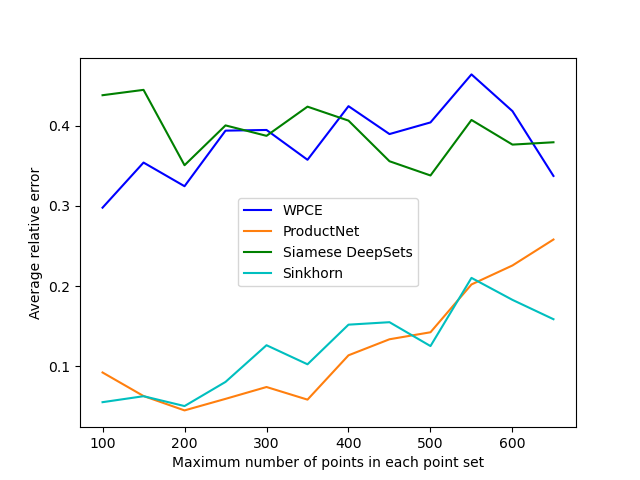}
         \caption{noisy-sphere-3.}
     \end{subfigure}
     \hfill
     \begin{subfigure}[b]{0.30\textwidth}
         \centering
         \includegraphics[width=\textwidth]{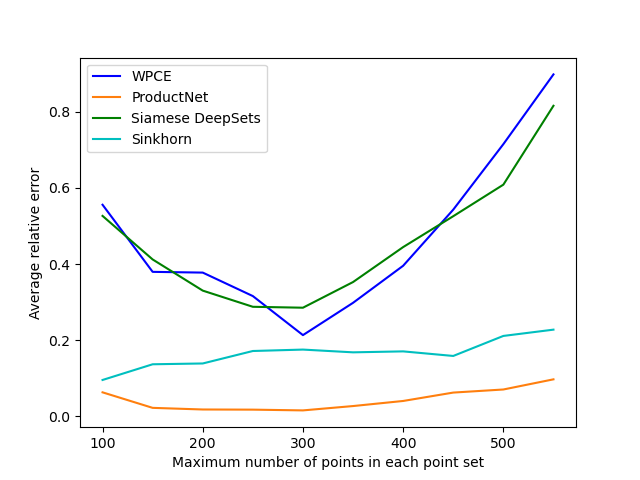}
         \caption{noisy-sphere-6.}
     \end{subfigure}
     \hfill
     \begin{subfigure}[b]{0.30\textwidth}
         \centering
         \includegraphics[width=\textwidth]{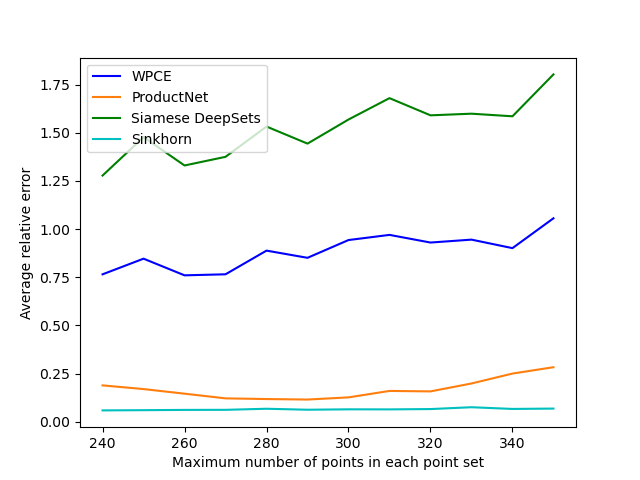}
         \caption{uniform.}
     \end{subfigure}
     \hfill
     \begin{subfigure}[b]{0.30\textwidth}
         \centering
         \includegraphics[width=\textwidth]{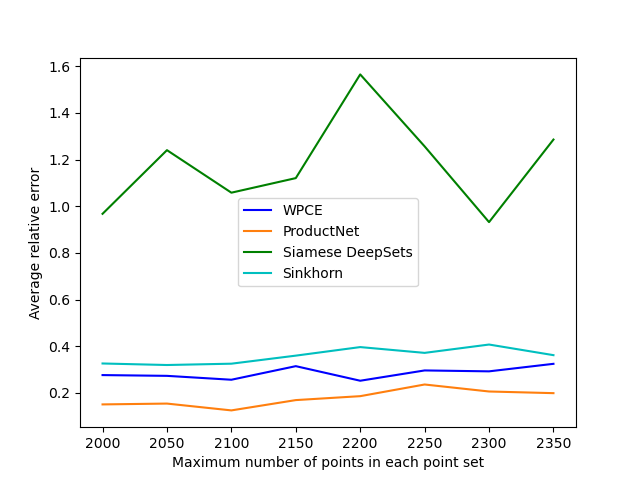}
         \caption{ModelNet-large.}
     \end{subfigure}
     \hfill
     \begin{subfigure}[b]{0.30\textwidth}
         \centering
         \includegraphics[width=\textwidth]{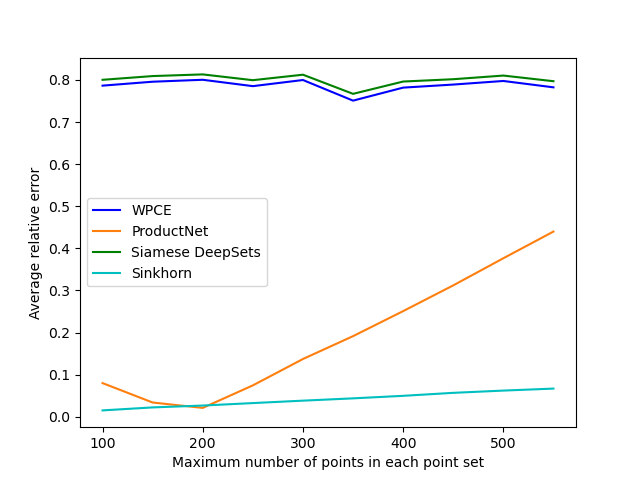}
         \caption{RNAseq.}
     \end{subfigure}
        \caption{Average 2-Wasserstein error for each model as the maximum number of points increases. Note that for all datasets, the Sinkhorn error is with $\epsilon = 0.10$.}
        \label{fig:generalization_graph}
\end{figure}

\section{Extra proofs}

\subsection{Proof of DeepSets Universality.}
\label{appendix:deepsets}
Here we will provide extra details on the proof of Theorem \ref{thm:sum-decomposition} using multisymmetric polynomials. Note that multisymmetric polynomials were previously used in  \citep{segol2019universal} and \citep{maron2019universality} to show universality for equivariant set networks and arbitrary $G$-invariant neural networks for any permutation group $G$.

\begin{proof} 
To begin, we will first define multisymmetric polynomials and power sum multisymmetric polynomials. Let $A[y_1, \dots, y_n]$ be the ring of polynomials in $n$ variables with coefficients in a ring $A$.

    \begin{definition}[Multisymmetric polynomials]
        The multisymmetric polynomials on the $n$ families of $k$ variables $\mathbf{x}_1 \dots, \mathbf{x}_n$ where $\mathbf{x}_i = (x_{i, 1}, \dots, x_{i, k})$ are those polynomials that remain unchanged under every permutation of the $n$ families, $\mathbf{x}_1, \dots, \mathbf{x}_n$.
    \end{definition}
    Let $A$ be a ring. The \textit{algebra of multisymmetric polynomials} in $n$ families of $k$ variables with coefficients in $A$ is denoted $\mathfrak{J}_n^k(A)$. 

    Furthermore, we define the multisymmetric power-sum polynomials:
    \begin{definition}[Multisymmetric power-sum polynomials]
        Let $\alpha = (\alpha_1, \dots, \alpha_k) \in \N^k$. Given $\mathbf{x} = (x_1, \dots, x_k)$, let
        \begin{equation*}
            \mathbf{x}^{\alpha} = x_1^{\alpha_1} \cdots x_k^{\alpha_k}
        \end{equation*}
        The multisymmetric power sum with multi-degree $\alpha$ is 
        \begin{equation*}
            p_{\alpha} = \sum_i^n \mathbf{x}_i^{\alpha}
        \end{equation*}
        Among them, we will consider the set of elementary multisymmetric power sums to be those with $|\alpha| \leq n$. 
    \end{definition}
    Notice that there $t = \binom{n + k}{k}$ multisymmetric power sums. Let $\alpha_1, \dots, \alpha_t$ be the list of all $\alpha \in \N^k$ such that $|\alpha| \leq n$. 
    \begin{theorem}[\citep{briand2004algebra}]
    \label{thm:multisymmetric-power-sums}
        Let $A$ be a ring in which $n!$ is invertible. Then the multisymmetric power sums generate $\mathfrak{J}_n^r(A)$ as an $A$-algebra. 
    \end{theorem}
    If we take $A = \R$, we get that the multisymmetric power sum polynomials generate $\mathfrak{J}_N^k$. 
    Now given a continuous function $f: \R^{k \times N} \to \R$ which is invariant to permutations of the columns, we know that $f$ can be approximated by a polynomial $p$ which is invariant to permutations of columns (see \citep{maron2019universality} for a detailed argument). Such a polynomial $p$ is a multisymmetric polynomial in $N$ families of $k$ variables with coefficients in $\R$ i.e., $p \in \mathfrak{J}_N^k$. Given $\mathbf{x} \in \R^k$,
    \begin{equation*}
        \phi(\mathbf{x}) = [x^{\alpha_1}, \dots, x^{\alpha_t}]
    \end{equation*}
    Then 
    \begin{equation*}
        \sum_{i = 1}^N \phi(\mathbf{x}_i) = \begin{bmatrix}
            \sum_i^N x_i^{\alpha_1}\\
            \sum_i^N x_i^{\alpha_2} \\
            \vdots \\
            \sum_i^N x_i^{\alpha_t}
        \end{bmatrix} = \begin{bmatrix}
            p_{\alpha_1} \\
            p_{\alpha_2} \\
            \vdots \\
            p_{\alpha_t}
        \end{bmatrix}
    \end{equation*}
    By Theorem \ref{thm:multisymmetric-power-sums}, we have that $p_{\alpha_1}, \dots, p_{\alpha_t}$ will generate any polynomial in $\mathfrak{J}_N^k$. Then we have some polynomial $q \in \R[y_1, \dots, y_t]$ such that $p = q(p_{\alpha_1}, \dots, p_{\alpha_t})$. 
\end{proof}

\subsection{Proofs from Section 3}
\label{appendix:section-3-proofs}
\subsubsection{Proof of Lemma 3.2}
\begin{proof}
     Let $\epsilon > 0$. 
     Since $f$ is uniformly continuous, $\exists \delta$ such that for all $(A_1, \dots, A_m), (A_1', \dots,  A_m') \in \mathcal{X}_1 \times \cdots \times \mathcal{X}_m$ where $\mathrm{d}_{\mathcal{X}_1 \times \cdots \times \mathcal{X}_k}((A_1, \dots, A_m), (A_1',\dots, A_m')) < \delta$, we have $|f(A_1, \dots, A_m) - f(A_1', \dots, A_m') | < \epsilon$.
     Since for any $\delta > 0$ and any $i \in [m]$, $(\mathcal{X}_i, \mathrm{d}_{\mathcal{X}_i})$ has a $(\delta, a, G)$-sketch, we know that there is an $a_i \in \N^+$ where there are continuous $h_i: \mathcal{X}_i \to \R^{a_i}$ and $g_i: \R^{a_i} \to \mathcal{X}_i$ where $\mathrm{d}_{\mathcal{X}_i}(g_i \circ h_i(A), A) < \delta/m$ for each $A \in \mathcal{X}_i$.

    Let $g': \R^{a_1} \times \cdots \times \R^{a_m} \to \mathcal{X}_i \times \cdots \times \mathcal{X}_m$ be defined as $(u_1, \dots, u_m) \mapsto (g_1(u_1), \dots, g_m(u_m))$. 
    Since $\mathrm{d}_{\mathcal{X}_i}(g_i \circ h_i(A_i), A_i) < \delta/k$ for $A_i \in \mathcal{X}_i$ and $i \in [m]$, 
    \begin{equation*}
        \mathrm{d}_{\mathcal{X}_1 \times \cdots \times \mathcal{X}_m}((g_1 \circ h_1(A_1), \dots,  g_m\circ h_m(A_m)), (A_1, \dots, A_m)) < \delta.
    \end{equation*}
    Let $\rho = f \circ g'$ and $\phi_i = h_i$ Then 
    \begin{align*}
        |f(A_1, \dots, A_m) - \rho(\phi_1(A_1), \dots, \phi_k(A_m))| &= |f(A_1, \dots, A_m) - f \circ g'(h_1(A_1), \dots, h_m(A_m))| \\
        &= |f(A_1, \dots, A_m) - f(g \circ h(A_1),\dots, g \circ h(A_m))| < \epsilon.
    \end{align*}
    Note that if $\mathcal{X}_1 = \mathcal{X}_2 = \cdots = \mathcal{X}_m$ and $G_1 = G_2 = \cdots = G_m$, we can use the same encoding and decoding function - $h_i$ and $g_i$, respectively - for all $\mathcal{X}_i$. 
    Thus, in this case, $\phi_1 = \phi_2 = \cdots = \phi_m$.
\end{proof}

\subsubsection{Proof of Lemma 3.3}
\begin{proof}
    Using the same argument as Theorem \ref{thm:product-network}, we know that for $\epsilon/2$, there is a continuous $h: \mathcal{X} \to \R^a$ and $g: \R^{a} \to \R$ such that 
    \begin{equation*}
        |f(A_1, \dots, A_m) - f(g \circ h(A_1), \dots, g \circ h (A_m))| < \frac{\epsilon}{2}.
    \end{equation*}
    As before, let $g': \R^{a \times m} \to \R$ be $g'(u_1, \dots, u_m) = (g(u_1), \dots, g(u_m))$.
    Take $F: \R^{a \times m} \to \R$ as $F(u_1,\dots, u_m) = f \circ g' (u_1, \dots, u_m) = f(g(u_1), \dots, g(u_m))$ where $u_i$ represents the $i$th column of an element in $\R^{a \times m}$. Note that $F$ is continuous and invariant to permutations of the columns. Let $t = \binom{a + m}{m}$. Therefore, by Theorem \ref{thm:sum-decomposition}, there is a $\gamma: \R^a \to \R^{t}$ and $\rho$ such that
    \begin{equation*}
        \Bigg|f\circ g'(v_1, \dots, v_m) - \rho\Big(\sum_{i = 1}^m \gamma(v_i) \Big) \Bigg| < \frac{\epsilon}{2}
    \end{equation*}

Now set $\phi = \gamma \circ h$ which is a function from $\R^a \to \R^t$. We thus have that
    \begin{align*}
       &\Big|f(A_1, \dots, A_m) - \rho\Big(\sum_{i = 1}^m \phi(A_i)\Big)\Big| \\
        &= \Bigg|f(A_1, \dots, A_m) - f \circ g'(h(A_1),\dots, h(A_m)) + f \circ g'(h(A_1),\dots,  h(A_m)) - \rho\Big(\sum_{i = 1}^m \phi(A_i)\Big)\Bigg| \\
        &< \frac{\epsilon}{2} + \frac{\epsilon}{2}= \epsilon
    \end{align*}
This completes the proof. 
\end{proof}

\end{document}